\documentclass{article}
\usepackage{PRIMEarxiv}
\usepackage[utf8]{inputenc} 
\usepackage[T1]{fontenc}    
\usepackage{hyperref}       
\usepackage{url}            
\usepackage{booktabs}       
\usepackage{amsfonts}       
\usepackage{nicefrac}       
\usepackage{microtype}      
\usepackage{subcaption}
\usepackage{authblk}
\usepackage{float}
\usepackage{amsthm}
\usepackage{amsmath}
\usepackage{amssymb}
\usepackage{program}
\usepackage{layouts}
\usepackage{fancyhdr}       
\usepackage{graphicx}       
\newtheorem{theorem}{Theorem}

\title{IC classifier: a classifier for 3D industrial components based on geometric prior using GNN}

\author[1,2]{Zipeng Lin}
\author[3,4,5]{\thanks{Please address all correspondences to this author (zhenguonie@tsinghua.edu.cn} Zhenguo Nie}

\affil[1]{University of California, Berkeley}
\affil[2]{Computer Science division at UC Berkeley}
\affil[3]{Department of Mechanical Engineering, Tsinghua University, Beijing, 100084, China}
\affil[4]{State Key Laboratory of Tribology in Advanced Equipment, Beijing, 100084, China}
\affil[5]{Beijing Key Lab of Precision/Ultra-precision Manufacturing Equipments and Control, Beijing, 100084, China}
\affil[1]{\textit {zp-l@berkeley.edu}}
\affil[3]{\textit {zhenguonie@tsinghua.edu.cn}}

\begin{document}
\maketitle

\begin{abstract}
   In this paper, we propose an approach to address the problem of classifying 3D industrial components by introducing a novel framework named IC-classifier (Industrial Component classifier). Our framework is designed to focus on the object's local and global structures, emphasizing the former by incorporating specific local features for embedding the model. By utilizing graphical neural networks and embedding derived from geometric properties, IC-classifier facilitates the exploration of the local structures of the object while using geometric attention for the analysis of global structures. Furthermore, the framework uses point clouds to circumvent the heavy computation workload. The proposed framework's performance is benchmarked against state-of-the-art models, demonstrating its potential to compete in the field.
\end{abstract}


\section{Introduction}
\label{sec:intro}

While deep learning has performed well on image processing, optimization ~\cite{nie2021topologygan}, and identification task, sometimes images convey more information to complete the task. For instance, it would be hard to identify a 3D object or do segmentation by just using one image. Therefore, much 3D data are captured from 3D objects, such as RGB-D cameras or lidar (Light Detection and Ranging). The data structure is in the form of point clouds, a group of coordinates where each point is expressed as 3D coordinates $x, y, z$, and its other features. In our data, we primarily use the 3D coordinate of points.

There are three main challenges to developing deep learning frameworks on the 3D point structure: First, point clouds are different from images since they do not have apparent spatial structures that are dense. A point cloud might be sparse in some places and dense in others. Therefore, the framework would require finding a proper, efficient representation and capturing dense information from the sparse point cloud. Secondly, no matter how the size of the 3D object changes or how much it rotates, the framework's output should be the same. Thus, the framework should follow size, translation, and rotation invariance. As several ways of representation of the 3D data are proposed, another challenge emerges: some representations and their corresponding framework, as introduced in the next section, would have too big a run-time and space complexity, which is not applicable. 

Although frameworks like PointNet have answered the above challenges well, performance would be a significant concern. The difficulty is maintaining a balance between global features and local features. An ideal classification method is to achieve both. The global feature helps to capture a general idea of the object, while the local feature is also crucial since it reveals detail about the object. When classifying objects that differ in detail, capturing the local feature would be crucial. Therefore, we propose our model that utilizes both global and local features. Beyond that, we also incorporate geometric features, computed by the geometric properties of the object, into our model to further capture the local geometric detail of the objects. 

The paper will proceed as follows: in the second section, we will introduce the past work of 3D object classification, including how prior works choose the different representations of 3D data and explain our choice. Then, we discuss the ideas of our framework, including using a graphical neural network, preprocessing the dataset using geometric features, and so on. After that, we will illustrate the performance of our model by showing the confusion matrix, AUC-ROC curve, and comparison with our baseline, the PointNet model.

\section{Previous work}
\label{sec:previous_work}

The progress of developing deep learning models for 3D object classification is closely related to how those models treat and view 3D models. Initially, researchers have developed different representations of 3D points and corresponding models since it became clear that it is hard to do the task conventionally: the data of 3D objects is irregular and has one more dimension than the 2D data. Therefore, traditional methods that work on 2D images, such as CNN, cut short for 3D tasks if applied directly. The following summarizes the attempts to view 3D data structures differently and their derived models.

\subsection{Converting 3D model to 2D} 
To make 3D classification similar to 2d image classification, some people use multi-view to represent the 3d model ~\cite{RN23}.
A group of 2D images represents a 3D model by taking pictures at different angles to represent a multi-view model. Several models exploit the features of those pictures to grasp the feature for 3D data. 
For example, Hang et al. developed a model (MVCNN) ~\cite{RN23} that passes through each picture to a convolution neural network (CNN) and aggregates all the features through another CNN. 
Another way to deal with Multi-view images is by Feng et al. (GVCNN) ~\cite{RN15}, which gives picture discrimination scores before processing. However, multi-view could be better since it requires an algorithm to take pictures from angles, and it the difficulty to capture all the features from just a few pictures. 

In addition to using multi-view representations of the data, some models also attempt to use projection on 3D models. For instance,  Zhu et al. ~\cite{RN75} proposed a framework that uses autoencoder, an encoder-decoder model with the same output shape as input, to classify 3D shapes by feature learning after projecting them into 2D space.

Converting 3D data into 2D is beneficial because commonly used machine learning models like CNN are good at processing 2D shapes. However, using finite screenshots of 3D data might only catch part of the picture of the 3D model, as there are infinitely many angles to take such screenshots, and, likely, those images only tell part of the story. Therefore, works that focus on the raw 3D data itself were presented.

\subsection{Other 3D representation}

Unlike methods that depend on multi-view representations of the object, volumetric-based methods try to depict the 3D structure through techniques such as vocalization. Voxels ~\cite{nie2019voxel} are like pixels but for three dimension objects. Methods to utilize voxels for 3D object recognition start with VoxNet ~\cite{RN78}, where the 3d relationship inside an object is defined through blocks. The model works on 3D objects with sparse voxel representation but only works well on dense 3D data since the computation workload would be too much, and there is no promised runtime. 

Another representation to solve this problem is OctNet ~\cite{RN11}, which recursively partitions a point cloud using a grid-octree structure. Octree defines the 3D relationship through neighboring blocks. Octree-based CNN for 3D classification is also presented by Wang et al. ~\cite{RN11}. In the paper, the model hierarchically partitions point clouds and then encodes each octree into a bit string (with limited length), thus reducing run time. 

However, there are still issues with the above models since the asymptotic run-time analysis for the model still shows that the model runs slowly, and the amount for computation grows especially fast when the size of the data set increases.

\subsection{Miscellaneous data representation}
There are several other less-used data representations for 3D data, and corresponding models were proposed, and they are summarized here to ensure the past literature review is complete. Klovov et al. proposed KD-net ~\cite{klokov2017escape} that treats the 3D data as k-d tree and then trains a neural network. In the k-d tree, the tree's leaf nodes are normalized coordinates of the 3D data, while the non-leaf nodes are calculated from its children nodes with MLPs which share parameters to boost efficiency. Zeng et al. ~\cite{zeng20183dcontextnet} proposed a model with a similar idea as the K-D tree model but aggregates information of the model from more levels.

Rc-net proposed by Xu et al. ~\cite{xu2014rc} used RNN to accomplish point cloud embeddings. Instead of projecting, the model partitions the space into parallel beams and processes them. The shortcoming of the model is that it requires to consider 3D features when projecting, like the models that consider the 2D representation of 3D data above. Li et al. proposed So-net ~\cite{li2018so} that uses a low-dimensional representation of data, called self-organizing maps, to process the data.

\subsection{Point cloud representation}

Due to Multi-view and volumetric representation limitations, the point cloud format is used as the training data. PointNet, one of the pioneering frameworks ~\cite{RN1}, discovered the use of MLP (multi-layer perceptron) to satisfy the need for point classification perfectly. PointNet's success is attributed to its permutation and size invariance when the properties of the point cloud change. The fundamental concept is to learn a "spatial encoding of each point" and then aggregate all individual point features to create a "global point cloud signature ."The development of PointNet has inspired the creation of several other models.

\subsection{Variants of PointNet}

\subsection{Graphical neural networks}

Viewing the point cloud from a different perspective, the points in the point cloud can be seen as nodes in a graph, and the edges can be defined as relationships between the points. Graph neural network techniques attempt to apply a filter on a graph through the nodes' properties and edges, similar to applying CNN on a graph. Simonovsky and Komodakis ~\cite{simonovsky2017dynamic} were among the first to develop a framework that treats each point as a vertex of the graph and applies filters around the neighbors of a point. The information from the neighbors is aggregated to create a coarse graph from the original one. Dynamic graph CNN (DGCNN) ~\cite{RN422} uses an MLP to implement edges convolution from the edges. The framework employs a channel-wise symmetric aggregation, EdgeConv, to dynamically update the graph after each network layer. Inspired by DGCNN's approach, Hassani and Haley\cite{hassani2019unsupervised} developed a multi-task method to learn shape features using an encoder from multi-scale graphs and a decoder to process three unsupervised tasks. ClusterNet ~\cite{chen2019clusternet} employs a proven rotation-invariant module to generate rotation-invariant features from each point by processing its neighbor and constructing a hierarchical structure of a point cloud, which is similar to the approaches of DGCNN and PointNet++. The approach uses edge labels to get a convolution-like operation more suitable for point clouds. Unlike images, the 3D point cloud can be rotated, and the classification model should classify the same object, making the structure rotationally invariant.

So far, the existing methods rely on prior information and do not attempt to find the geometric properties of the point cloud. Even if some of the previous models try to get the local features of the 3D data, it still tries to do it by applying a model to the smaller portion of the data without understanding the information. By applying model-distilled data to the local 3D data, the model tries to find a pattern without understanding it. However, for some particular kinds of 3D data, like Industrial components,  there is detailed information in each category of such data that is only sometimes present in other categories. For example, some kinds of industrial components would display special geometric features.

We deem such information necessary for identifying objects like industrial parts. Therefore, inspired by other papers emphasizing the importance of treating the geometry of input \cite{nie2020stress}, we propose the following model, called IC-classifier, that tries to take in geometric priors and achieves better classification accuracy. From a higher point of view, our model displays the importance of including priors relevant to the dataset to improve efficiency.

\section{Data preparation}
We use the dataset from the mechanical parts benchmark generated by Kim, Chi, Hu, and Ramani ~\cite{sangpil2020large}. The raw data, however, have different points for different objects, so we will process that in the following sections to make the number of points in each model equal.

There are 67 categories of parts and 58696 objects in total.

\begin{figure}
     \centering
      \begin{subfigure}[b]{0.10\textwidth}
         \centering
         \includegraphics[width=\textwidth]{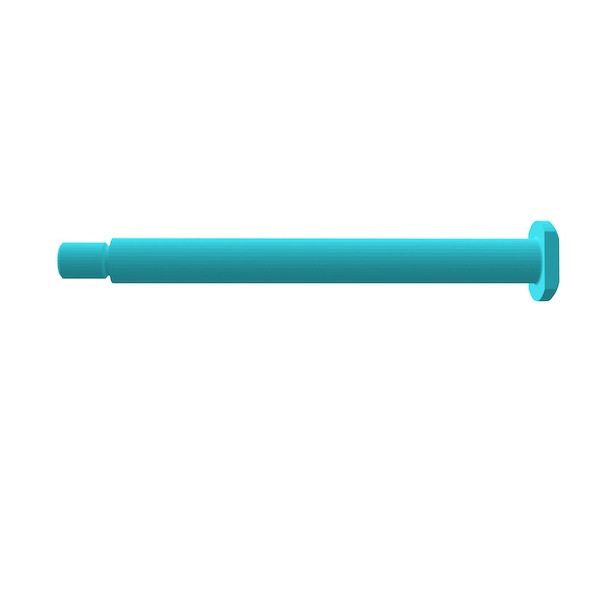}
     \end{subfigure}
     \begin{subfigure}[b]{0.10\textwidth}
         \centering
         \includegraphics[width=\textwidth]{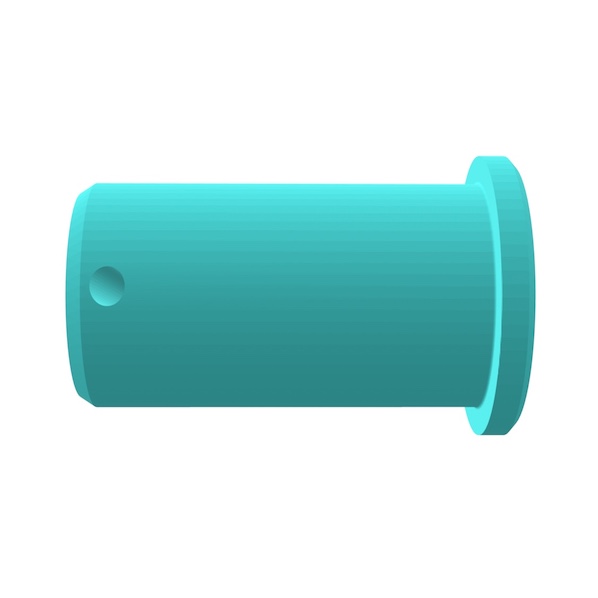}
     \end{subfigure}
     \begin{subfigure}[b]{0.10\textwidth}
         \centering
         \includegraphics[width=\textwidth]{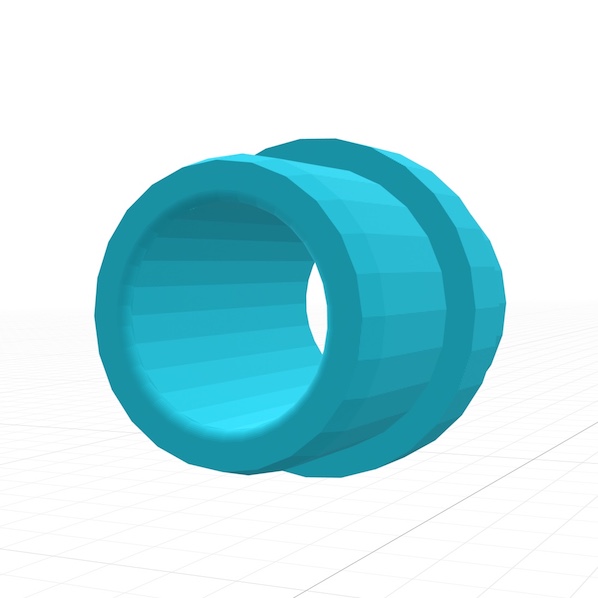}
     \end{subfigure}
     \begin{subfigure}[b]{0.10\textwidth}
         \centering
         \includegraphics[width=\textwidth]{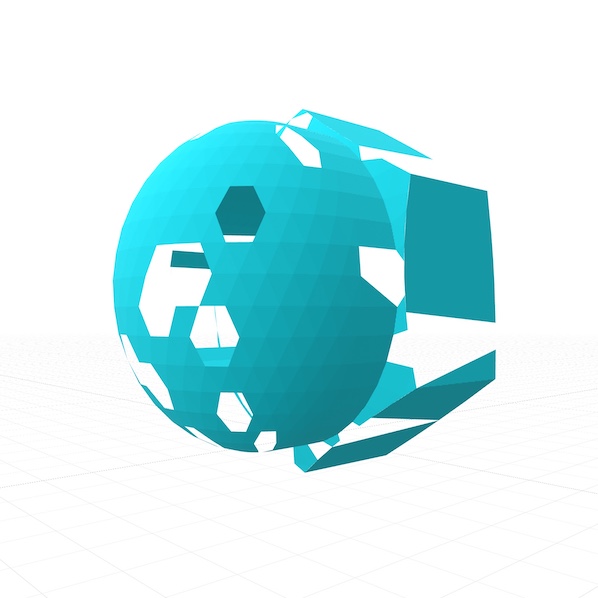}
     \end{subfigure}
     \begin{subfigure}[b]{0.10\textwidth}
         \centering
         \includegraphics[width=\textwidth]{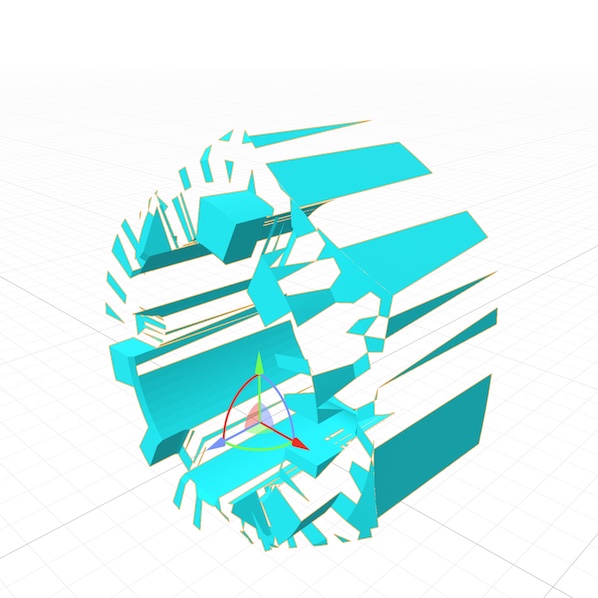}
     \end{subfigure}
    \begin{subfigure}[b]{0.10\textwidth}
         \centering
         \includegraphics[width=\textwidth]{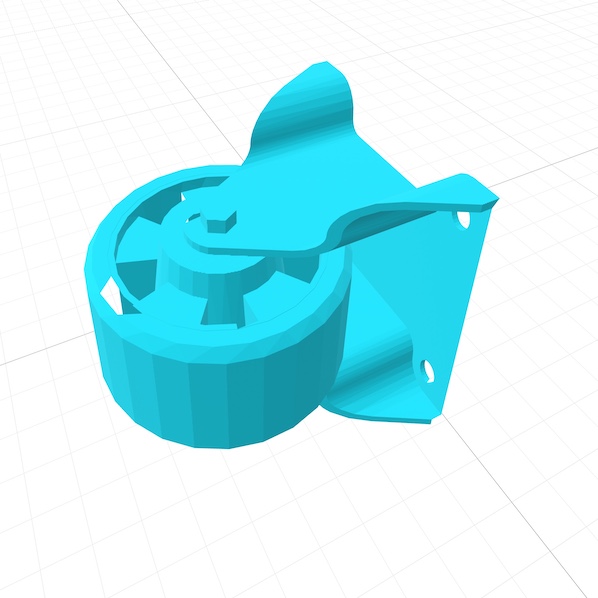}
     \end{subfigure}
    \begin{subfigure}[b]{0.10\textwidth}
         \centering
         \includegraphics[width=\textwidth]{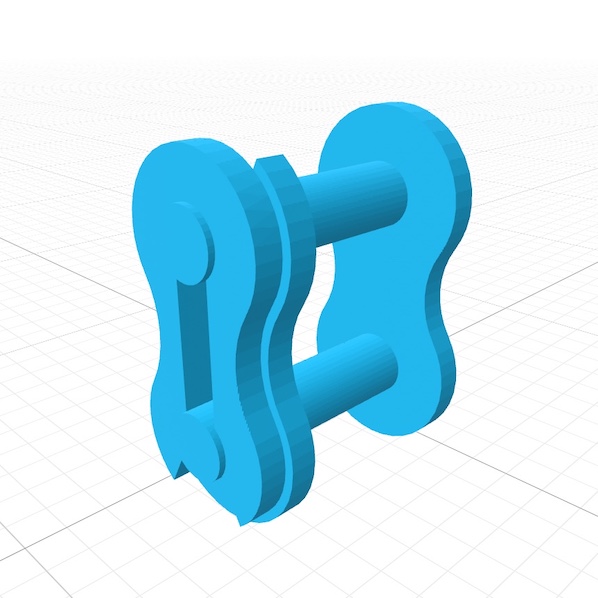}
     \end{subfigure}
    \begin{subfigure}[b]{0.10\textwidth}
         \centering
         \includegraphics[width=\textwidth]{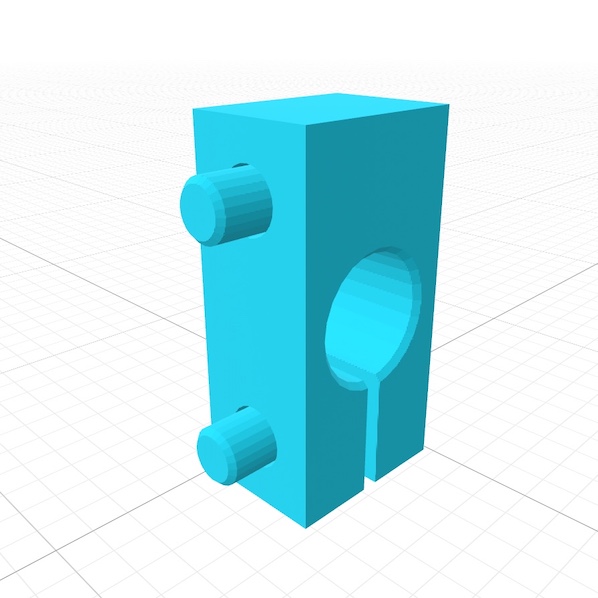}
     \end{subfigure}
    \begin{subfigure}[b]{0.10\textwidth}
         \centering
         \includegraphics[width=\textwidth]{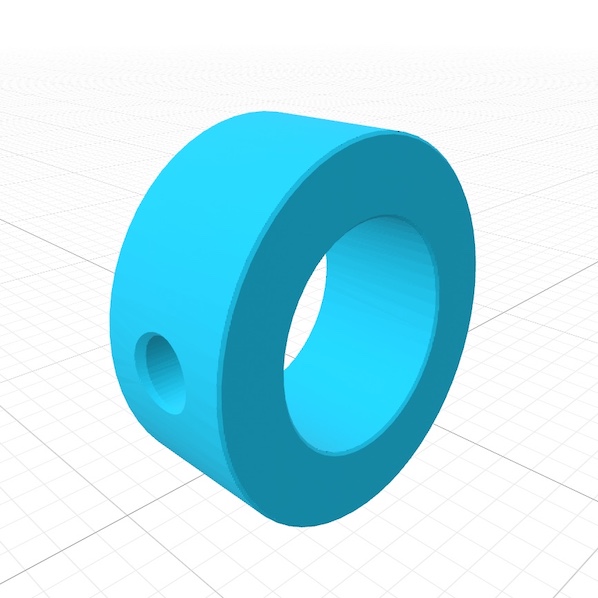}
     \end{subfigure}   
     \\
    \begin{subfigure}[b]{0.10\textwidth}
         \centering
         \includegraphics[width=\textwidth]{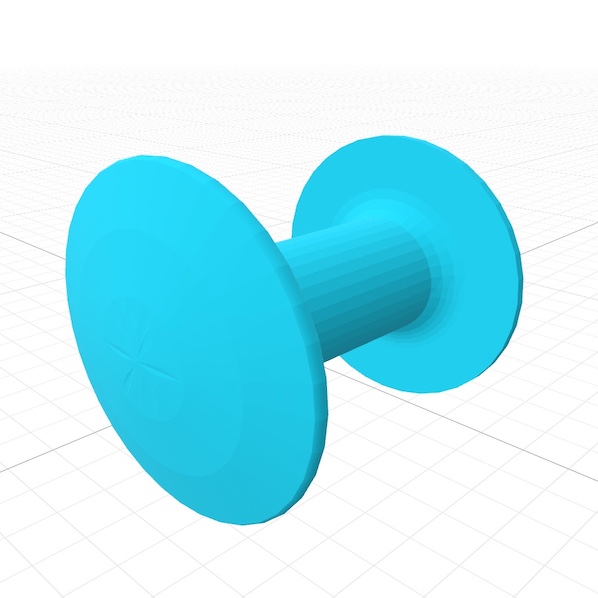}
     \end{subfigure}     
    \begin{subfigure}[b]{0.10\textwidth}
         \centering
         \includegraphics[width=\textwidth]{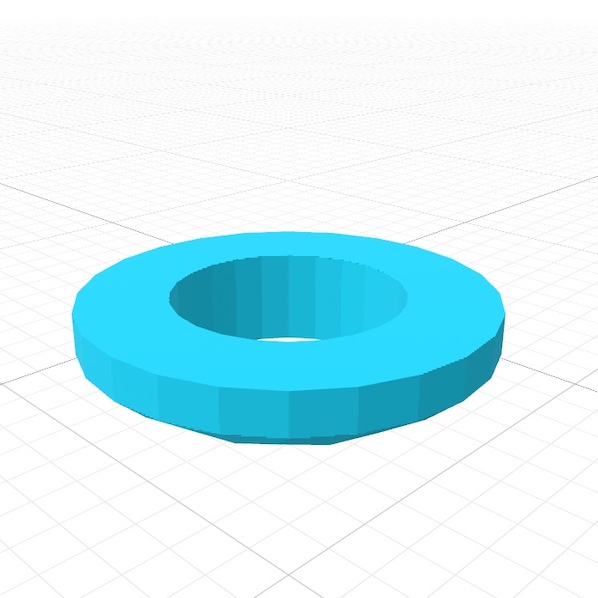}
     \end{subfigure}     
          \begin{subfigure}[b]{0.10\textwidth}
         \centering
         \includegraphics[width=\textwidth]{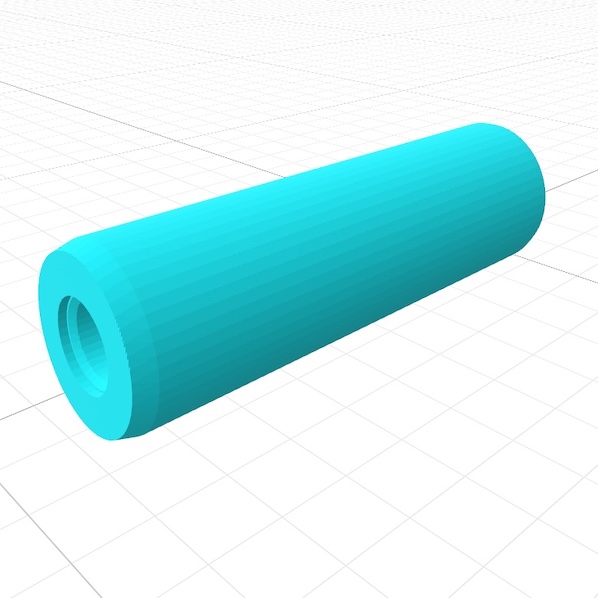}
     \end{subfigure}
     \begin{subfigure}[b]{0.10\textwidth}
         \centering
         \includegraphics[width=\textwidth]{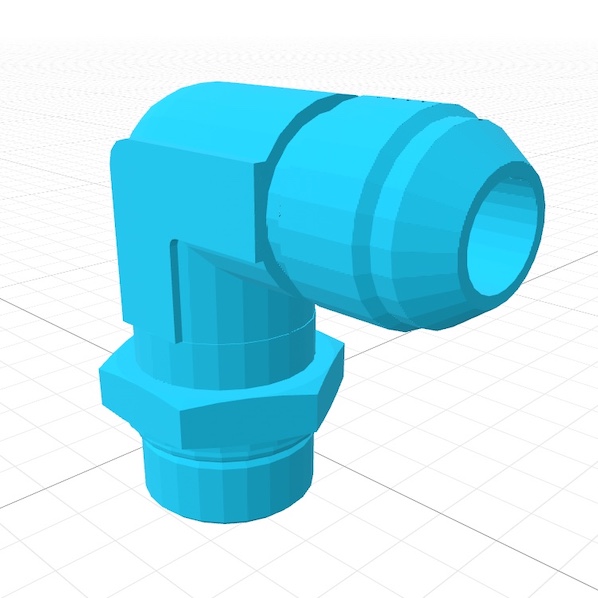}
     \end{subfigure}
     \begin{subfigure}[b]{0.10\textwidth}
         \centering
         \includegraphics[width=\textwidth]{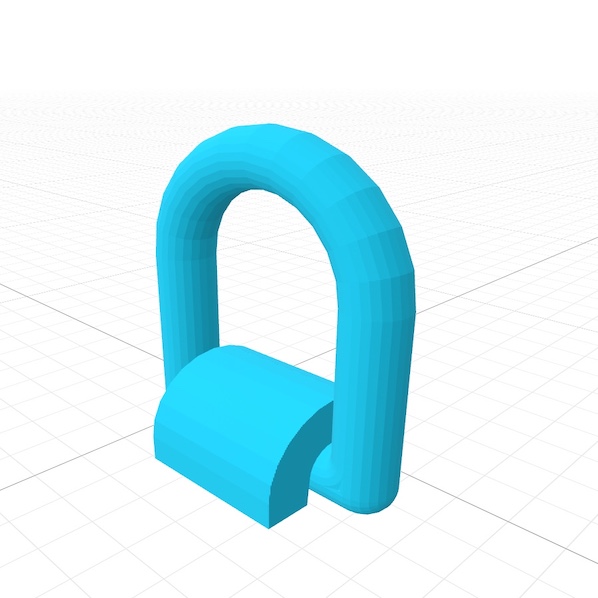}
     \end{subfigure}
     \begin{subfigure}[b]{0.10\textwidth}
         \centering
         \includegraphics[width=\textwidth]{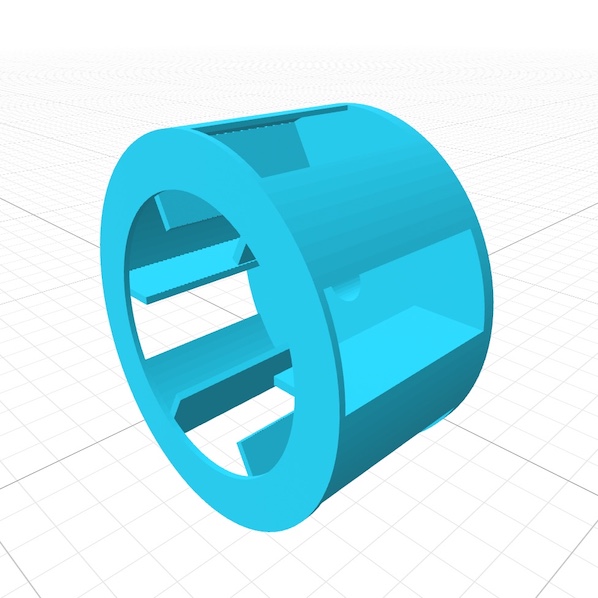}
     \end{subfigure}
    \begin{subfigure}[b]{0.10\textwidth}
         \centering
         \includegraphics[width=\textwidth]{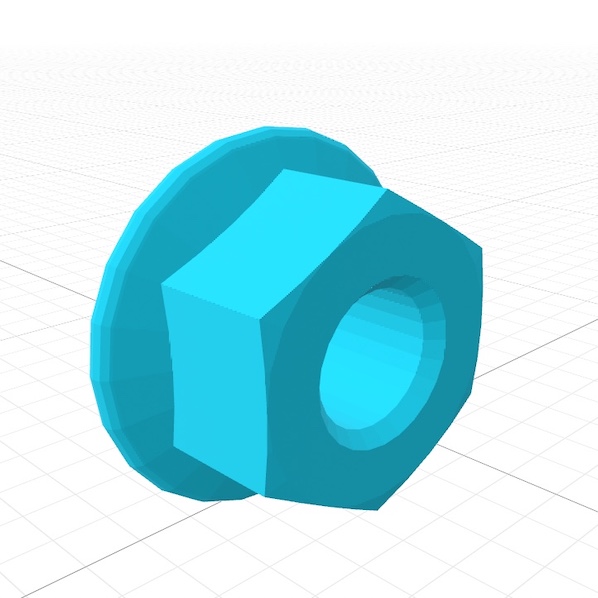}
     \end{subfigure}
    \begin{subfigure}[b]{0.10\textwidth}
         \centering
         \includegraphics[width=\textwidth]{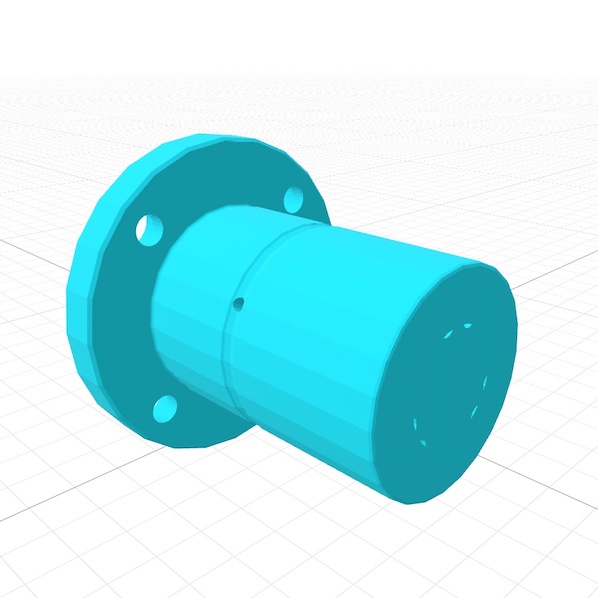}
     \end{subfigure}
    \begin{subfigure}[b]{0.10\textwidth}
         \centering
         \includegraphics[width=\textwidth]{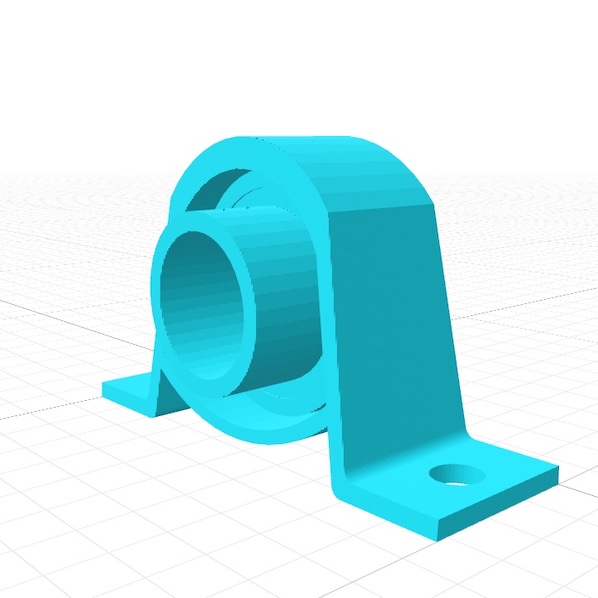}
     \end{subfigure}
     \\
    \begin{subfigure}[b]{0.10\textwidth}
         \centering
         \includegraphics[width=\textwidth]{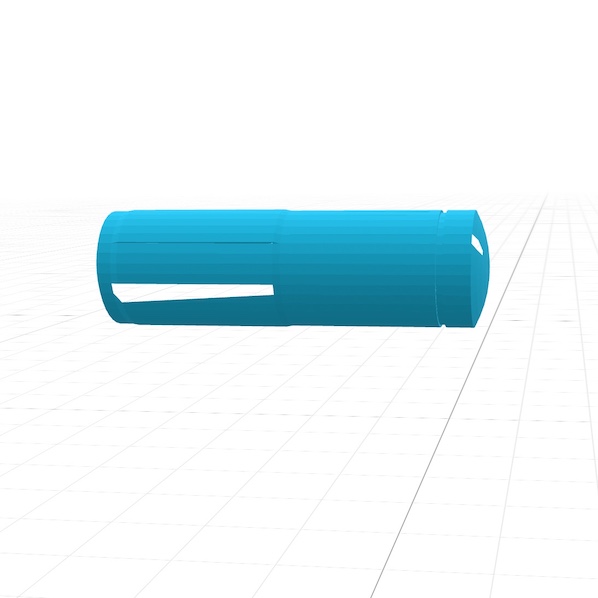}
     \end{subfigure}   
    \begin{subfigure}[b]{0.10\textwidth}
         \centering
         \includegraphics[width=\textwidth]{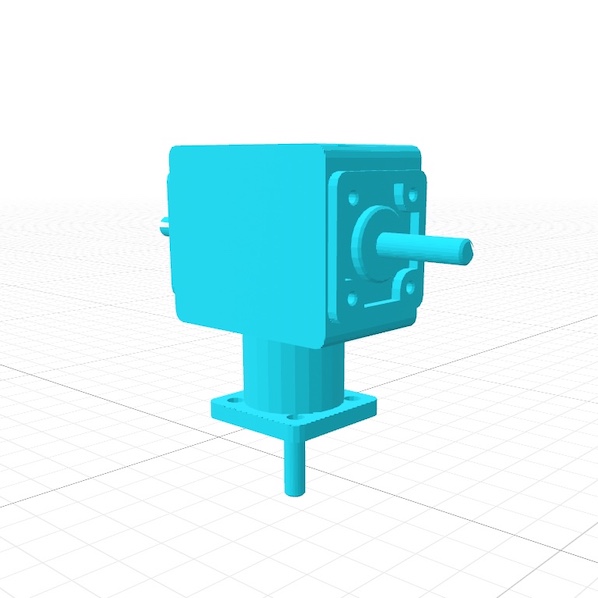}
     \end{subfigure}     
    \begin{subfigure}[b]{0.10\textwidth}
         \centering
         \includegraphics[width=\textwidth]{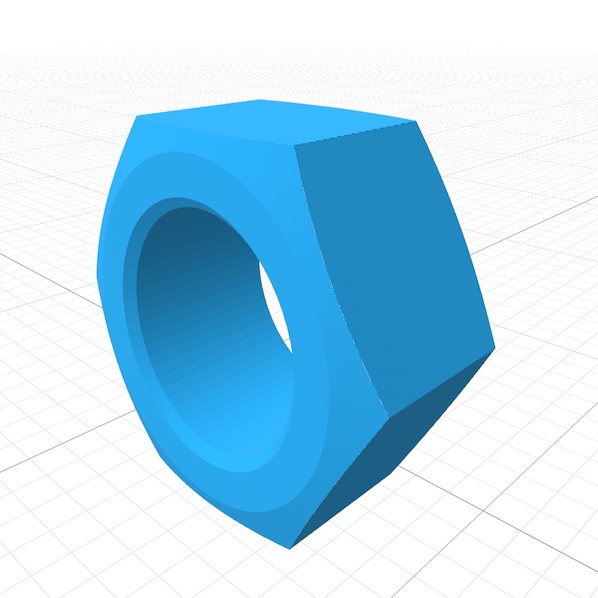}
     \end{subfigure}     
          \begin{subfigure}[b]{0.10\textwidth}
         \centering
         \includegraphics[width=\textwidth]{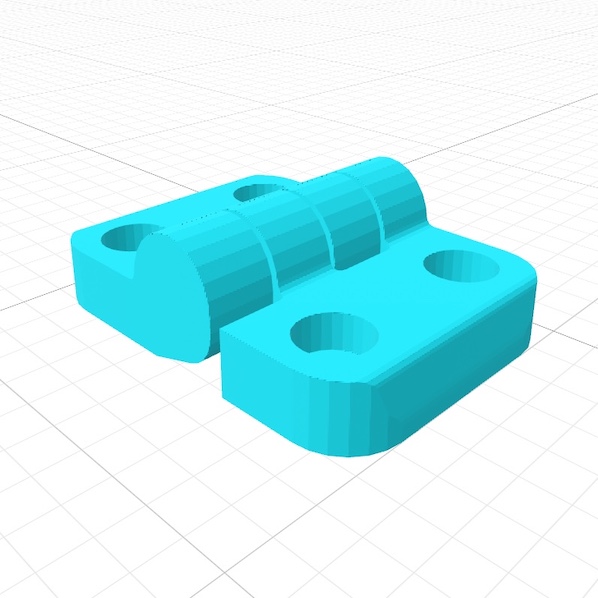}
     \end{subfigure}
     \begin{subfigure}[b]{0.10\textwidth}
         \centering
         \includegraphics[width=\textwidth]{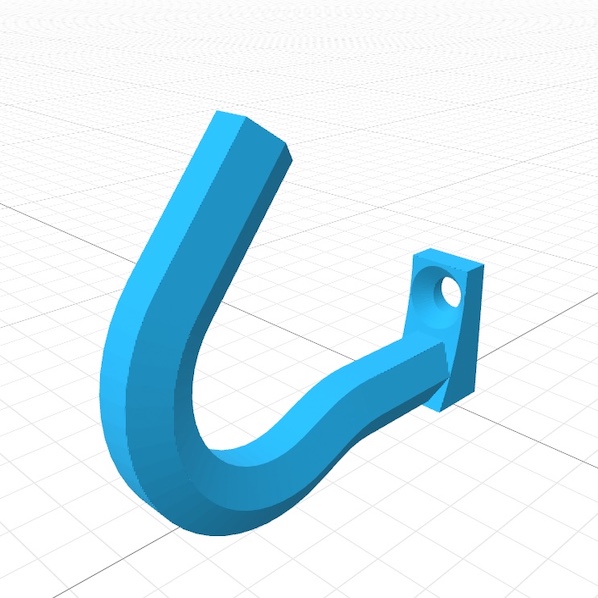}
     \end{subfigure}
     \begin{subfigure}[b]{0.10\textwidth}
         \centering
         \includegraphics[width=\textwidth]{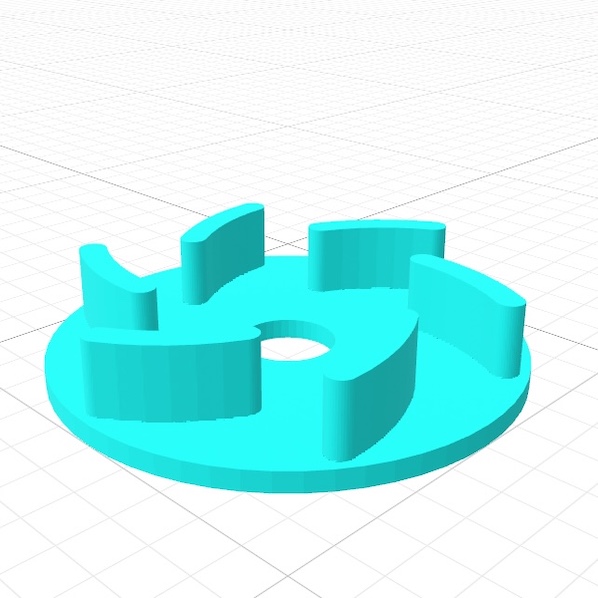}
     \end{subfigure}
     \begin{subfigure}[b]{0.10\textwidth}
         \centering
         \includegraphics[width=\textwidth]{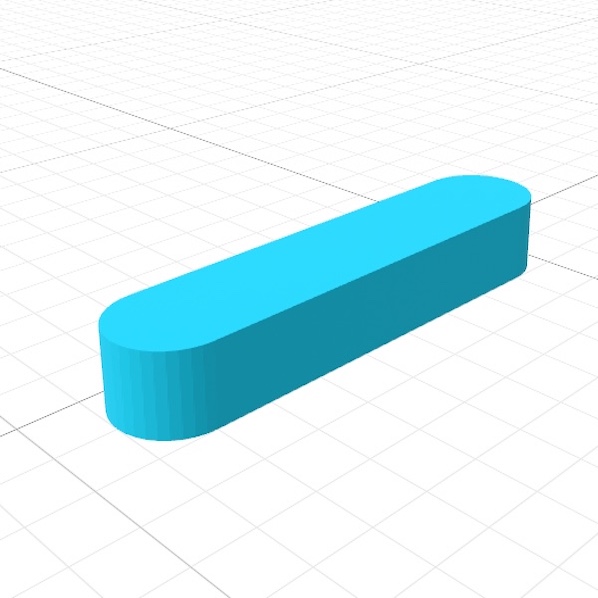}
     \end{subfigure}
    \begin{subfigure}[b]{0.10\textwidth}
         \centering
         \includegraphics[width=\textwidth]{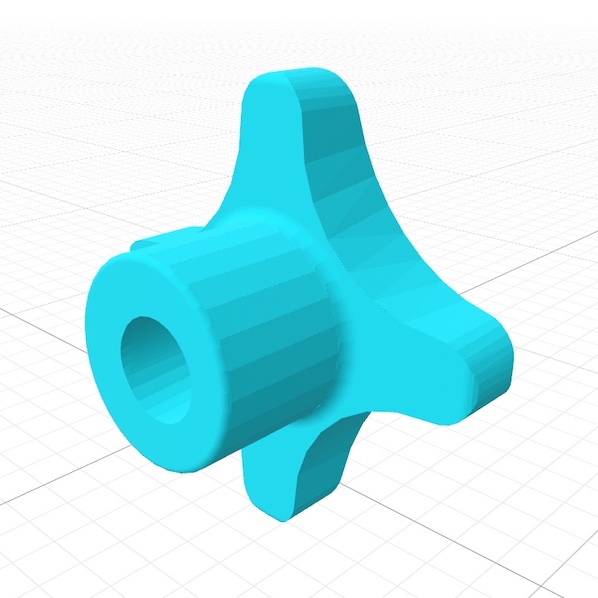}
     \end{subfigure}
    \begin{subfigure}[b]{0.10\textwidth}
         \centering
         \includegraphics[width=\textwidth]{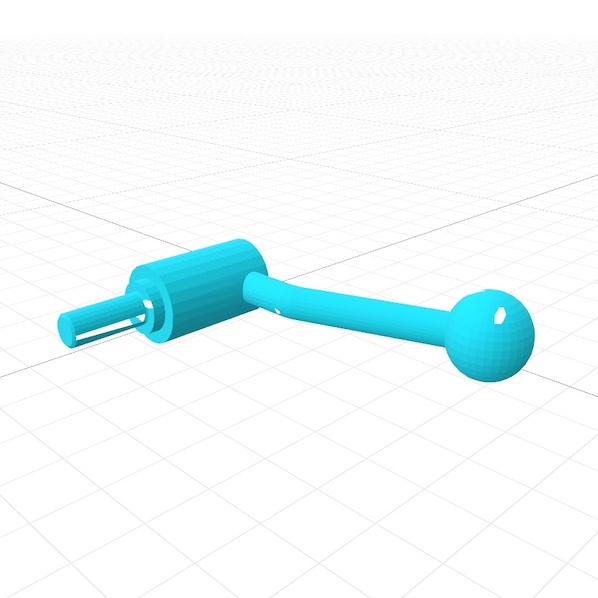}
     \end{subfigure}
     \\
    \begin{subfigure}[b]{0.10\textwidth}
         \centering
         \includegraphics[width=\textwidth]{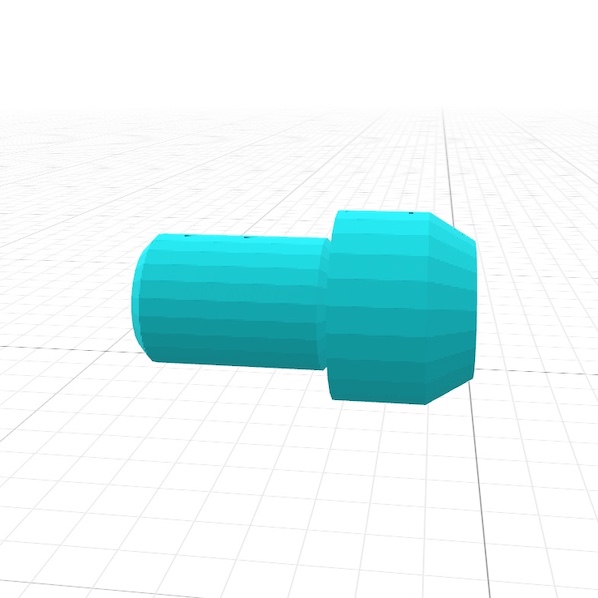}
     \end{subfigure}
    \begin{subfigure}[b]{0.10\textwidth}
         \centering
         \includegraphics[width=\textwidth]{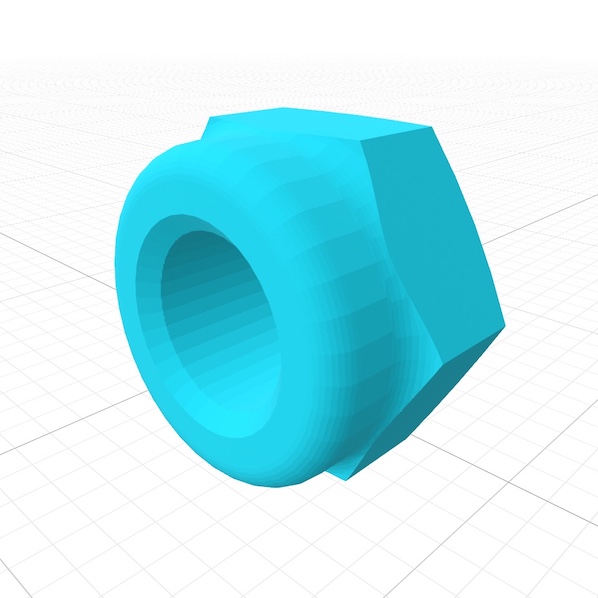}
     \end{subfigure}   
    \begin{subfigure}[b]{0.10\textwidth}
         \centering
         \includegraphics[width=\textwidth]{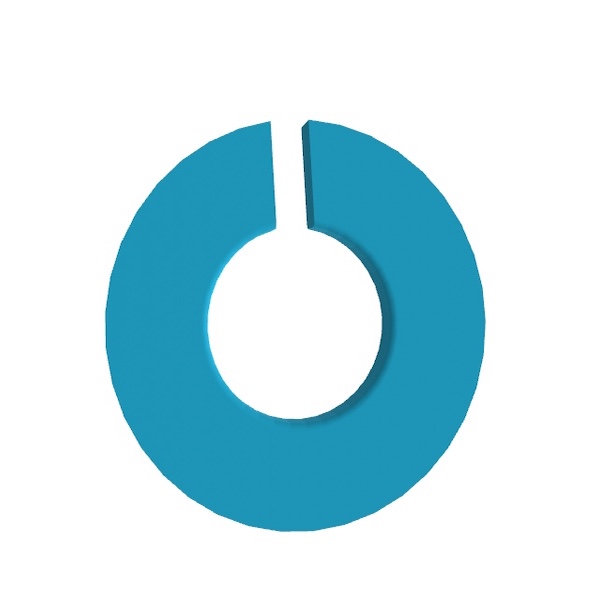}
     \end{subfigure}     
    \begin{subfigure}[b]{0.10\textwidth}
         \centering
         \includegraphics[width=\textwidth]{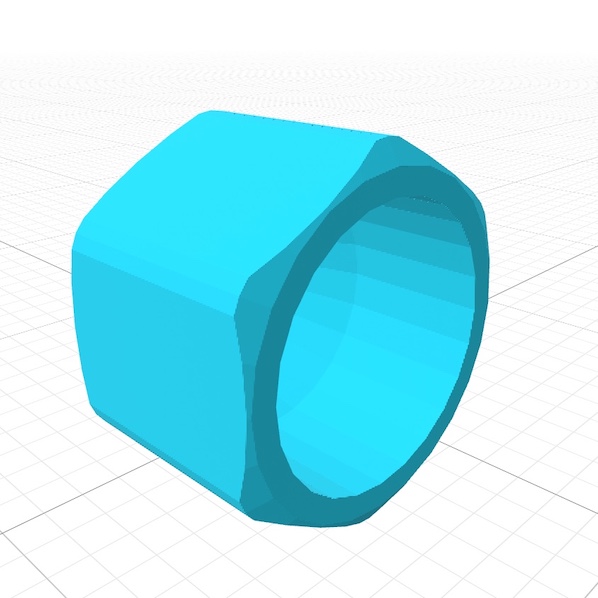}
     \end{subfigure}     
          \begin{subfigure}[b]{0.10\textwidth}
         \centering
         \includegraphics[width=\textwidth]{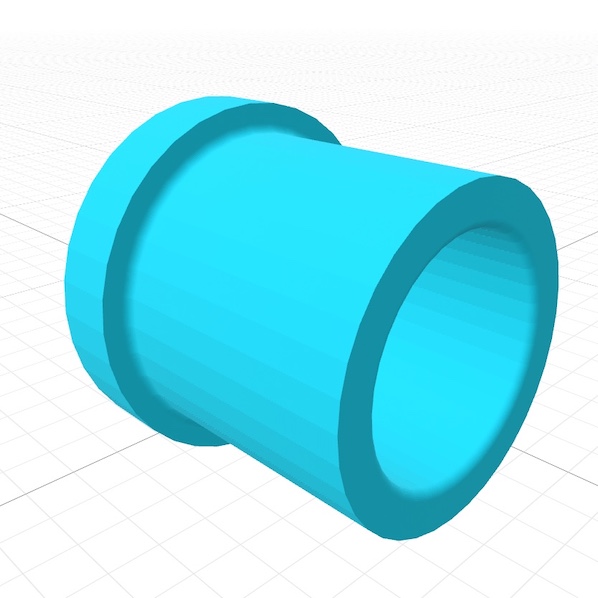}
     \end{subfigure}
     \begin{subfigure}[b]{0.10\textwidth}
         \centering
         \includegraphics[width=\textwidth]{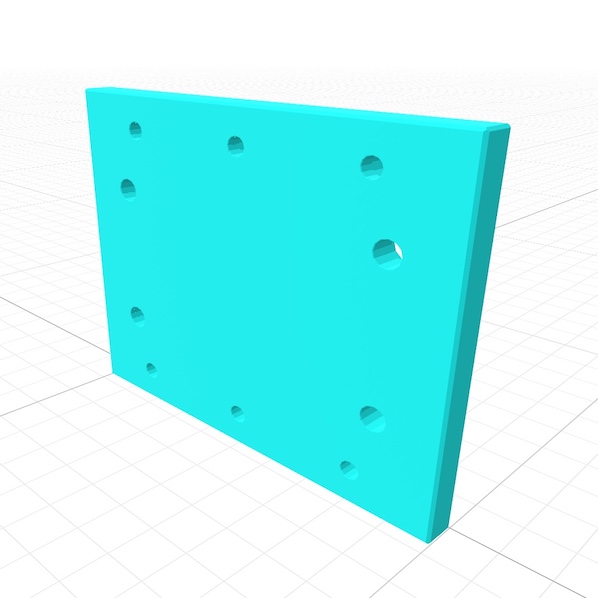}
     \end{subfigure}
     \begin{subfigure}[b]{0.10\textwidth}
         \centering
         \includegraphics[width=\textwidth]{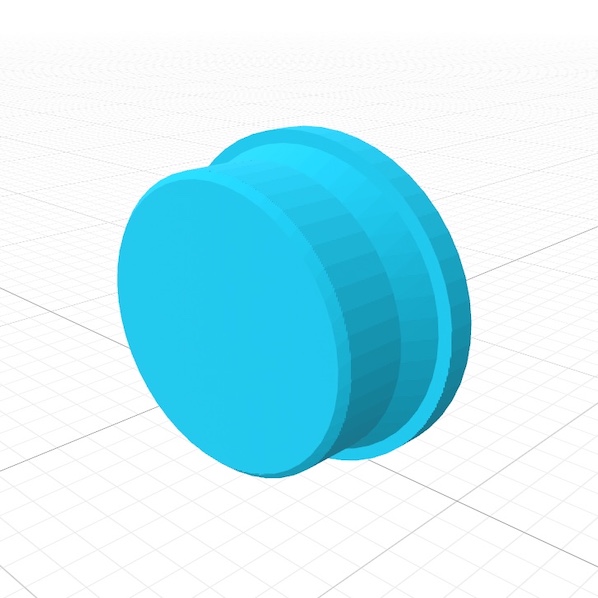}
     \end{subfigure}
     \begin{subfigure}[b]{0.10\textwidth}
         \centering
         \includegraphics[width=\textwidth]{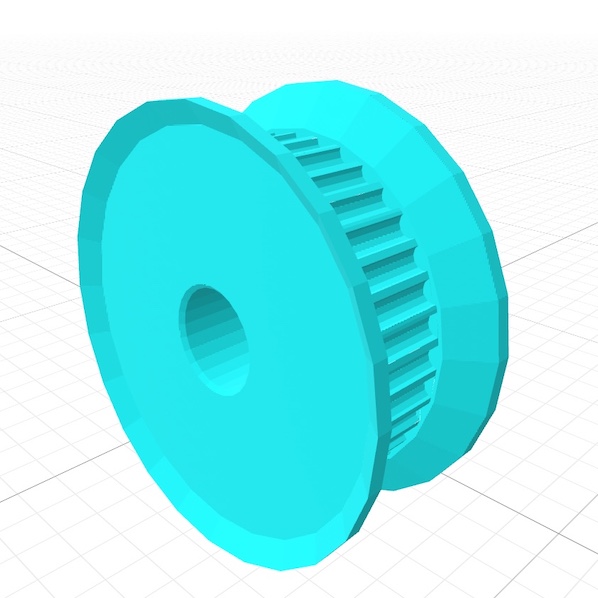}
     \end{subfigure}
    \begin{subfigure}[b]{0.10\textwidth}
         \centering
         \includegraphics[width=\textwidth]{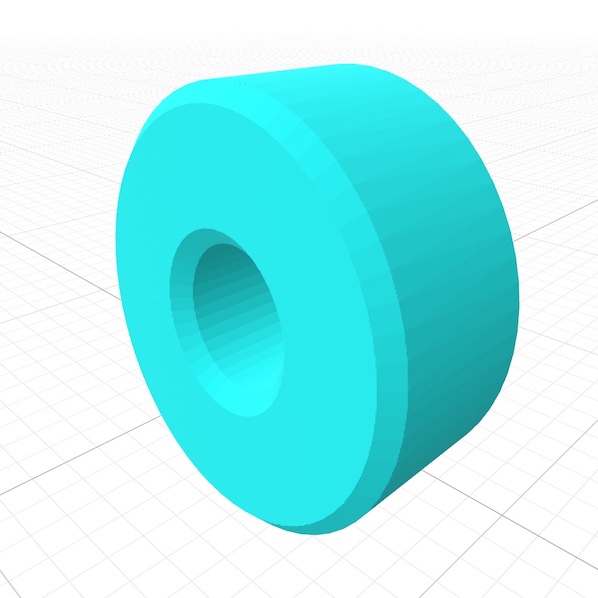}
     \end{subfigure}
     \\
    \begin{subfigure}[b]{0.10\textwidth}
         \centering
         \includegraphics[width=\textwidth]{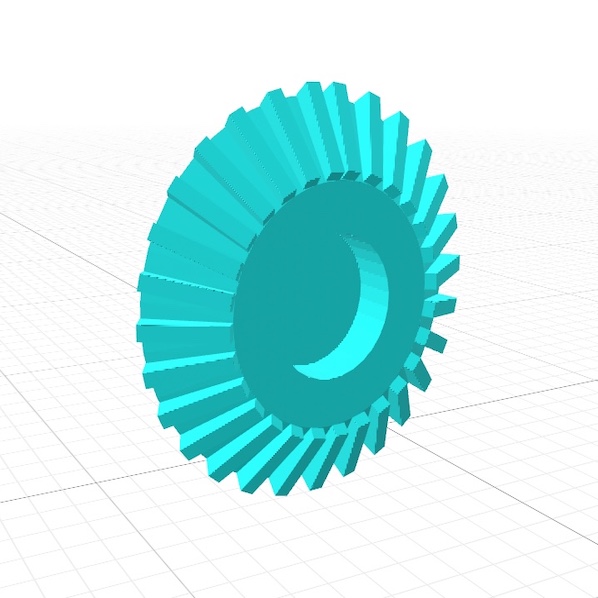}
     \end{subfigure}
    \begin{subfigure}[b]{0.10\textwidth}
         \centering
         \includegraphics[width=\textwidth]{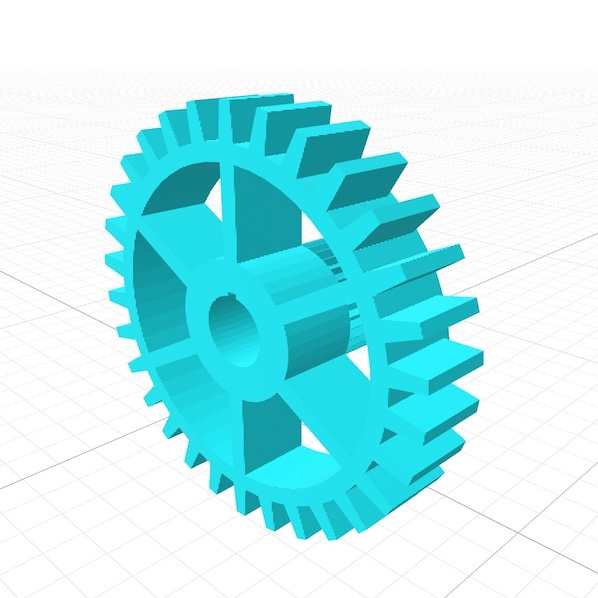}
     \end{subfigure}
    \begin{subfigure}[b]{0.10\textwidth}
         \centering
         \includegraphics[width=\textwidth]{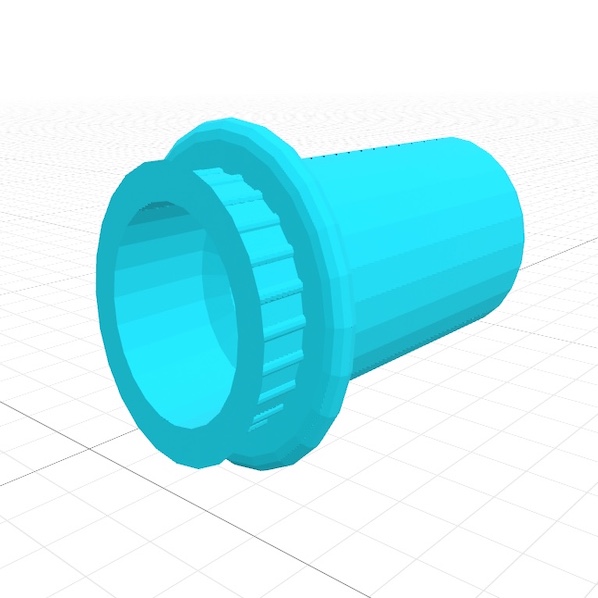}
     \end{subfigure}   
    \begin{subfigure}[b]{0.10\textwidth}
         \centering
         \includegraphics[width=\textwidth]{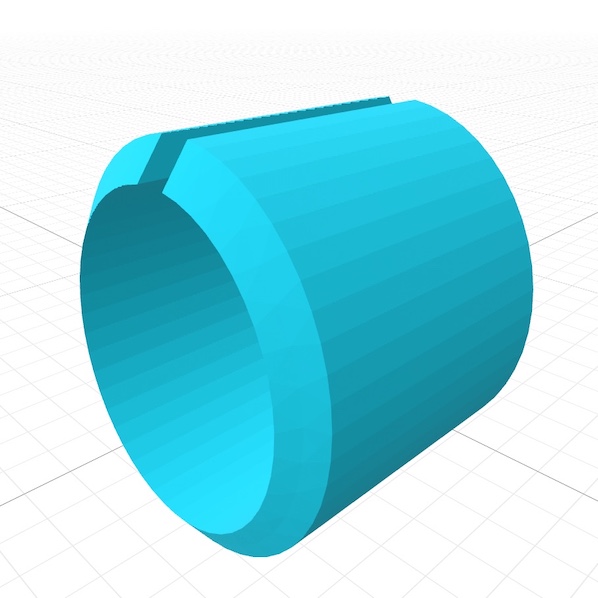}
     \end{subfigure}    
    \begin{subfigure}[b]{0.10\textwidth}
         \centering
         \includegraphics[width=\textwidth]{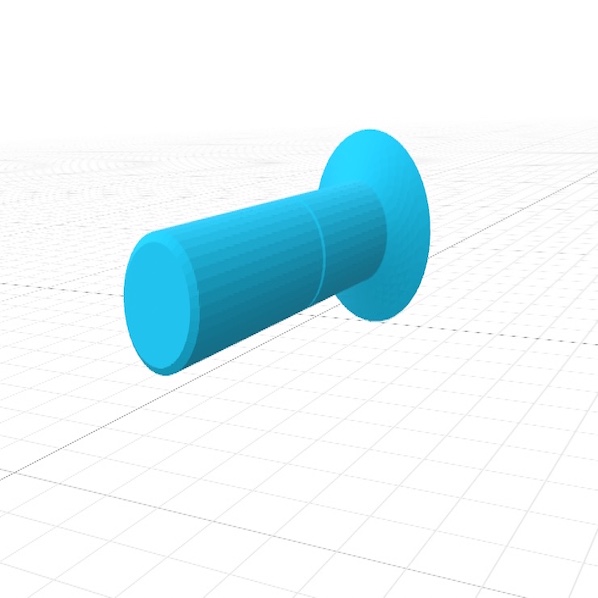}
     \end{subfigure}     
          \begin{subfigure}[b]{0.10\textwidth}
         \centering
         \includegraphics[width=\textwidth]{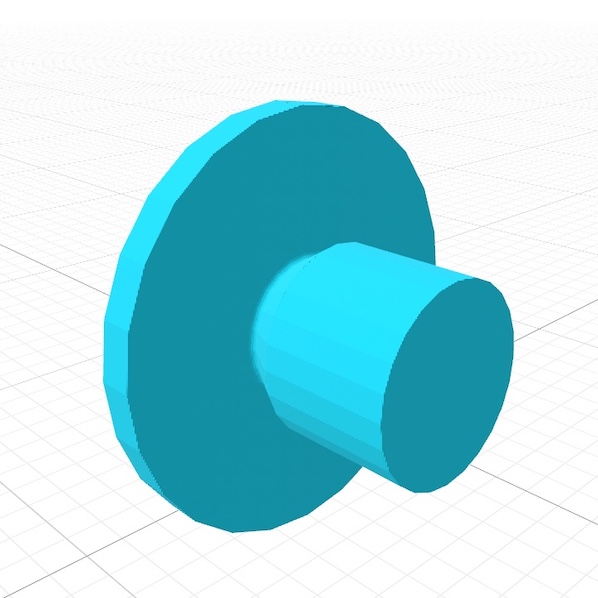}
     \end{subfigure}
     \begin{subfigure}[b]{0.10\textwidth}
         \centering
         \includegraphics[width=\textwidth]{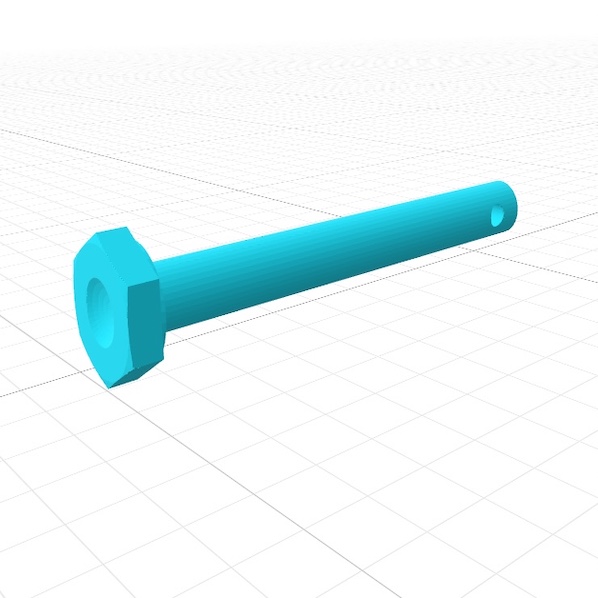}
     \end{subfigure}
     \begin{subfigure}[b]{0.10\textwidth}
         \centering
         \includegraphics[width=\textwidth]{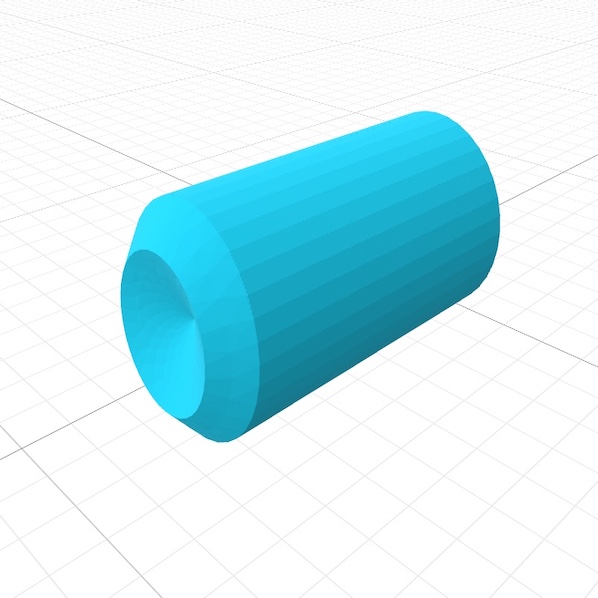}
     \end{subfigure}
     \begin{subfigure}[b]{0.10\textwidth}
         \centering
         \includegraphics[width=\textwidth]{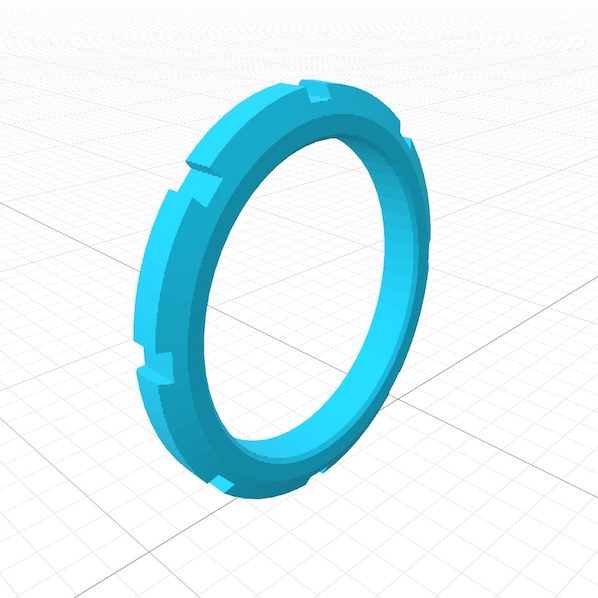}
     \end{subfigure}
     \\
    \begin{subfigure}[b]{0.10\textwidth}
         \centering
         \includegraphics[width=\textwidth]{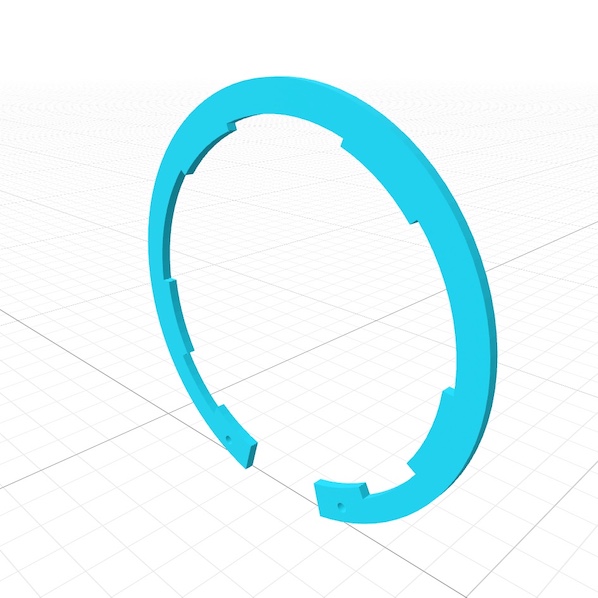}
     \end{subfigure}
    \begin{subfigure}[b]{0.10\textwidth}
         \centering
         \includegraphics[width=\textwidth]{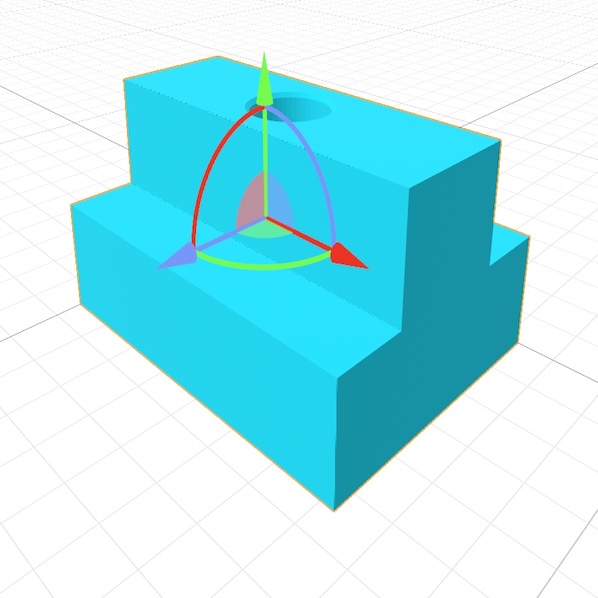}
     \end{subfigure}
    \begin{subfigure}[b]{0.10\textwidth}
         \centering
         \includegraphics[width=\textwidth]{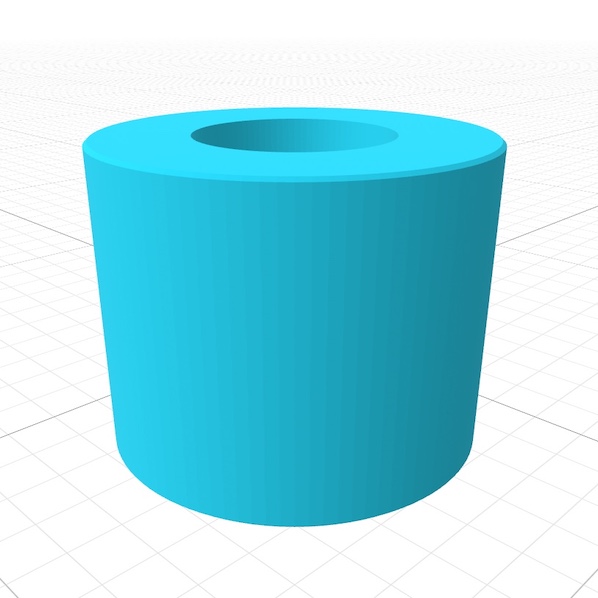}
     \end{subfigure}
    \begin{subfigure}[b]{0.10\textwidth}
         \centering
         \includegraphics[width=\textwidth]{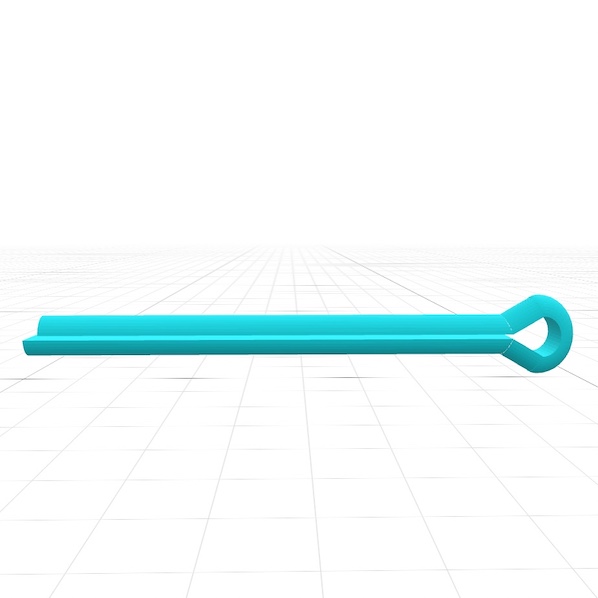}
     \end{subfigure}   
    \begin{subfigure}[b]{0.10\textwidth}
         \centering
         \includegraphics[width=\textwidth]{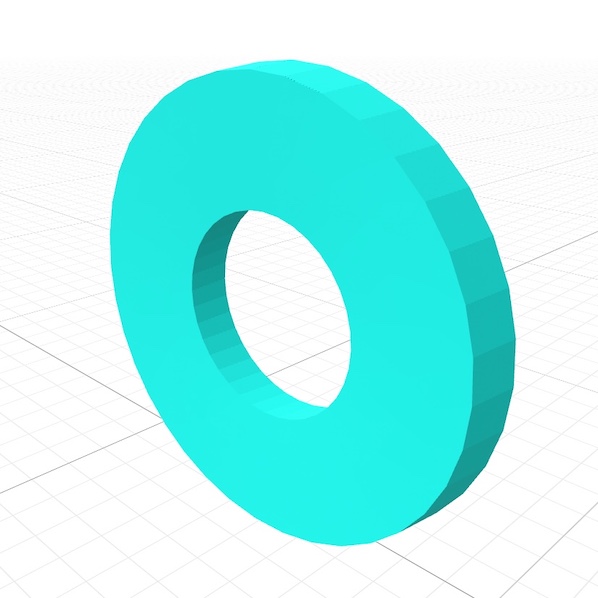}
     \end{subfigure}     
    \begin{subfigure}[b]{0.10\textwidth}
         \centering
         \includegraphics[width=\textwidth]{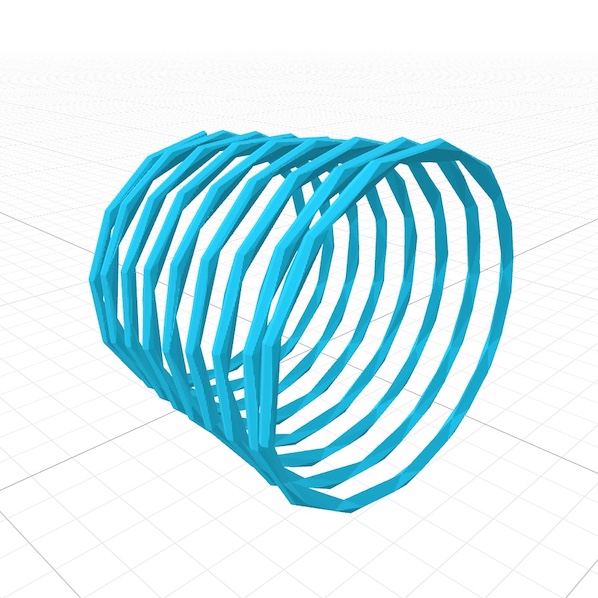}
     \end{subfigure}     
          \begin{subfigure}[b]{0.10\textwidth}
         \centering
         \includegraphics[width=\textwidth]{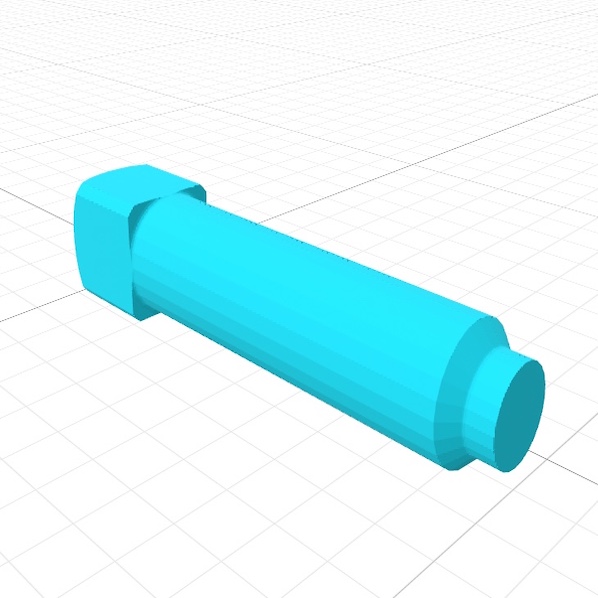}
     \end{subfigure}
     \begin{subfigure}[b]{0.10\textwidth}
         \centering
         \includegraphics[width=\textwidth]{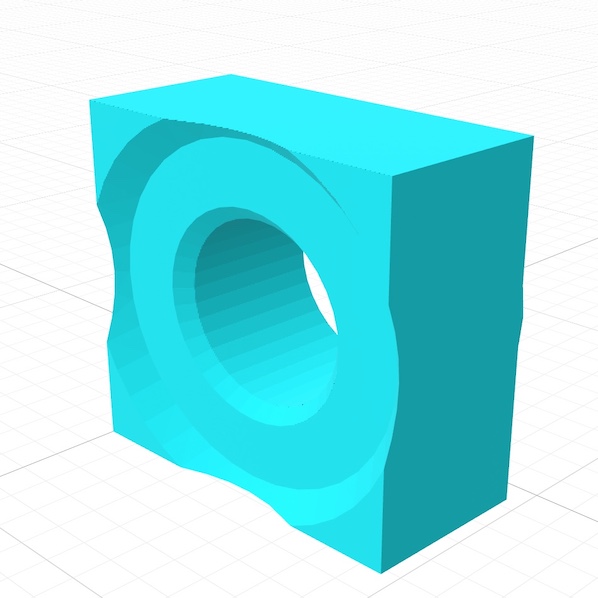}
     \end{subfigure}
     \begin{subfigure}[b]{0.10\textwidth}
         \centering
         \includegraphics[width=\textwidth]{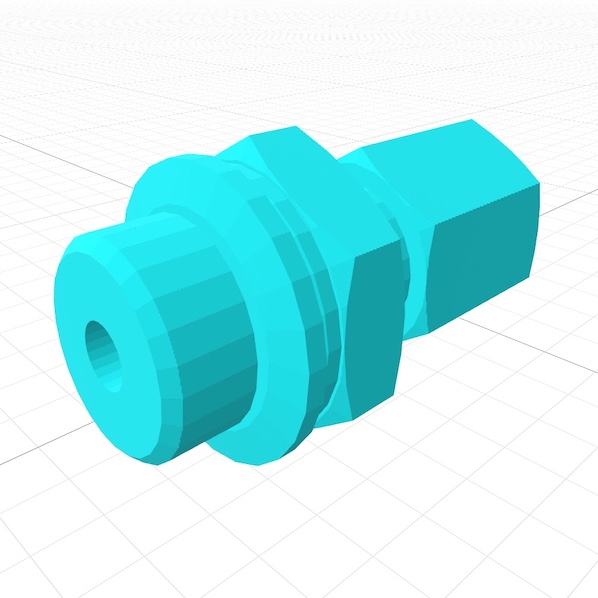}
     \end{subfigure}
     \\
     \begin{subfigure}[b]{0.10\textwidth}
         \centering
         \includegraphics[width=\textwidth]{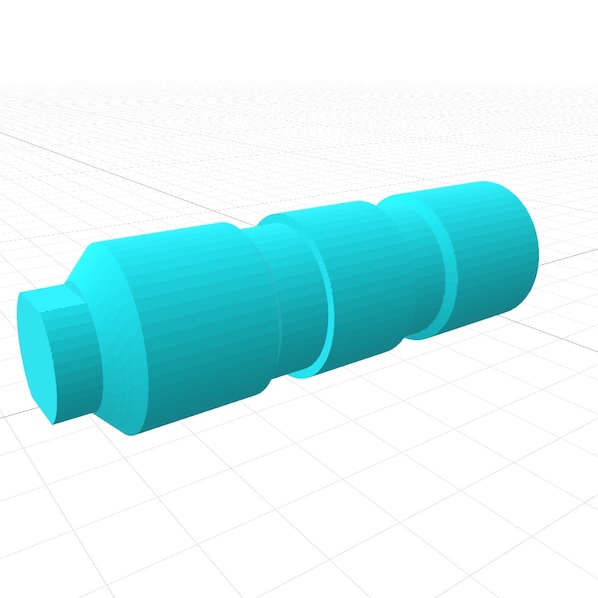}
     \end{subfigure}
    \begin{subfigure}[b]{0.10\textwidth}
         \centering
         \includegraphics[width=\textwidth]{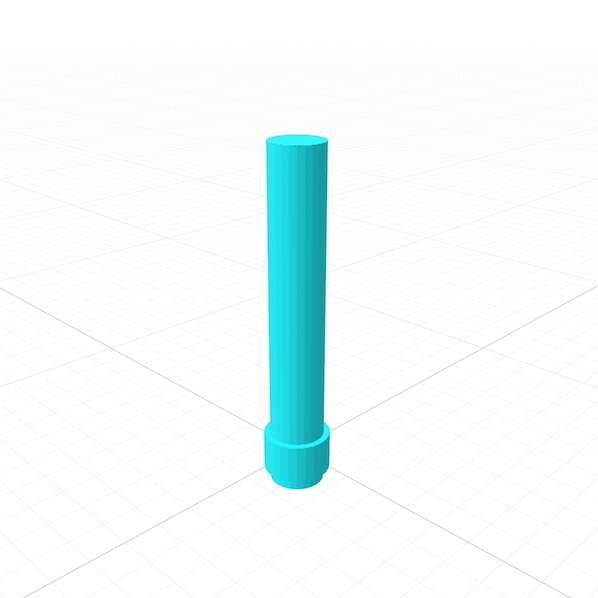}
     \end{subfigure}
    \begin{subfigure}[b]{0.10\textwidth}
         \centering
         \includegraphics[width=\textwidth]{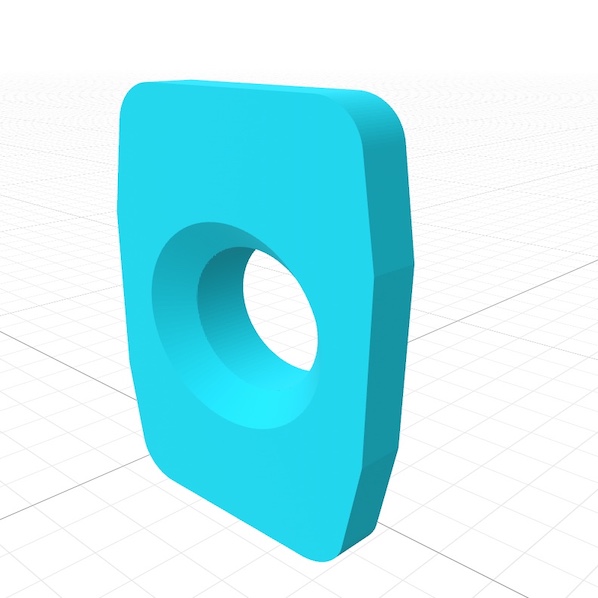}
     \end{subfigure}
    \begin{subfigure}[b]{0.10\textwidth}
         \centering
         \includegraphics[width=\textwidth]{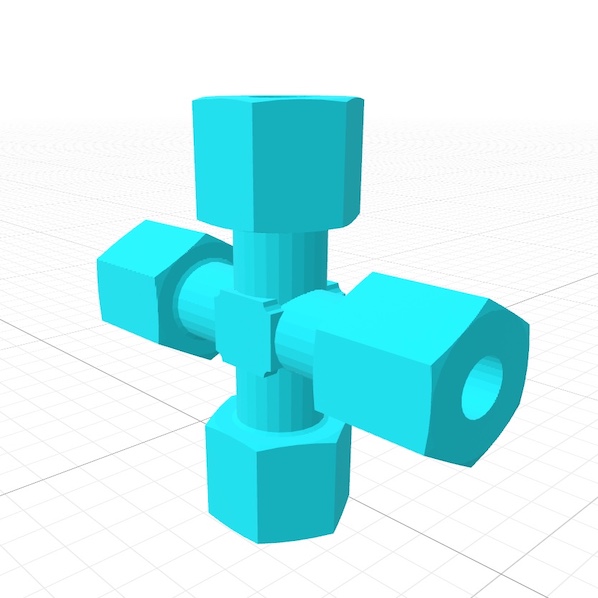}
     \end{subfigure}
    \begin{subfigure}[b]{0.10\textwidth}
         \centering
         \includegraphics[width=\textwidth]{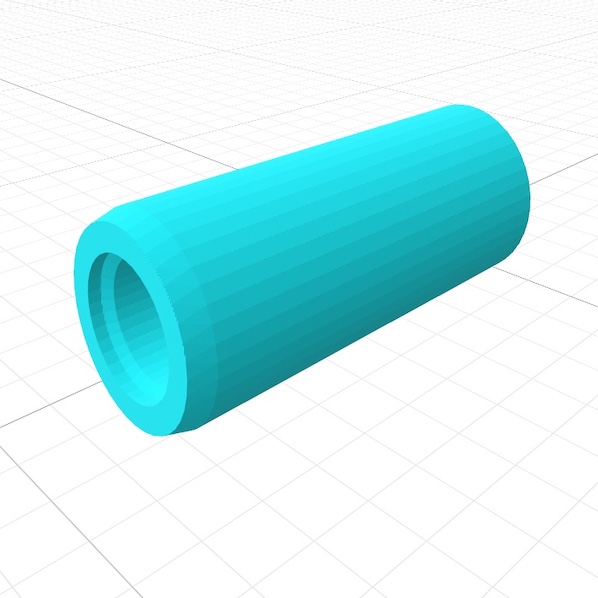}
     \end{subfigure}   
    \begin{subfigure}[b]{0.10\textwidth}
         \centering
         \includegraphics[width=\textwidth]{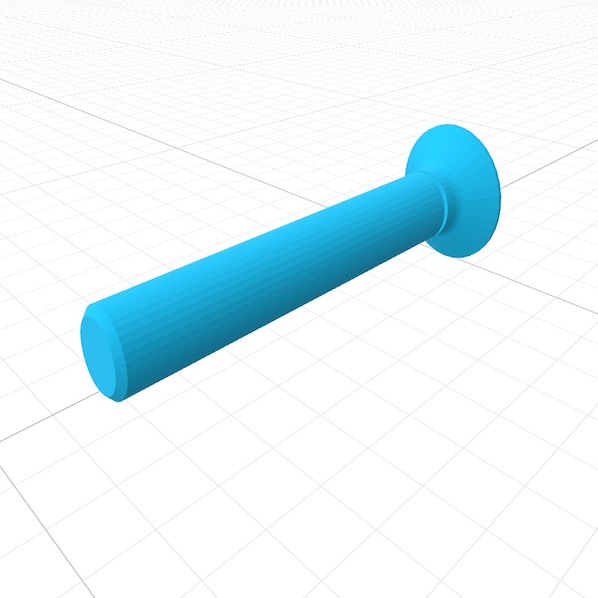}
     \end{subfigure}     
    \begin{subfigure}[b]{0.10\textwidth}
         \centering
         \includegraphics[width=\textwidth]{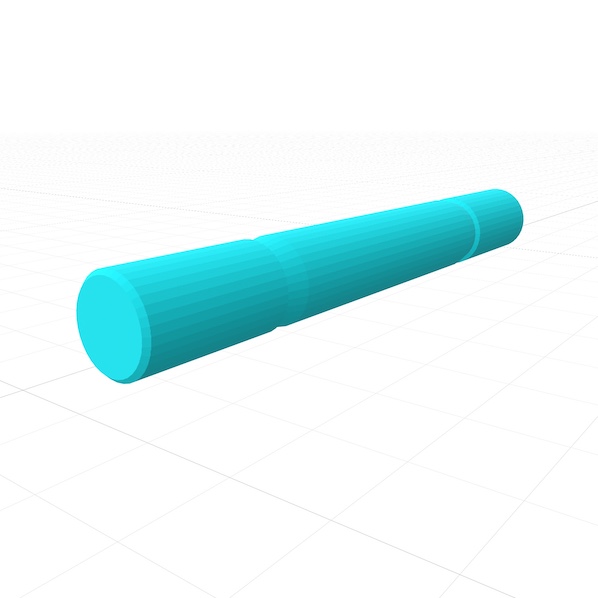}
     \end{subfigure}     
          \begin{subfigure}[b]{0.10\textwidth}
         \centering
         \includegraphics[width=\textwidth]{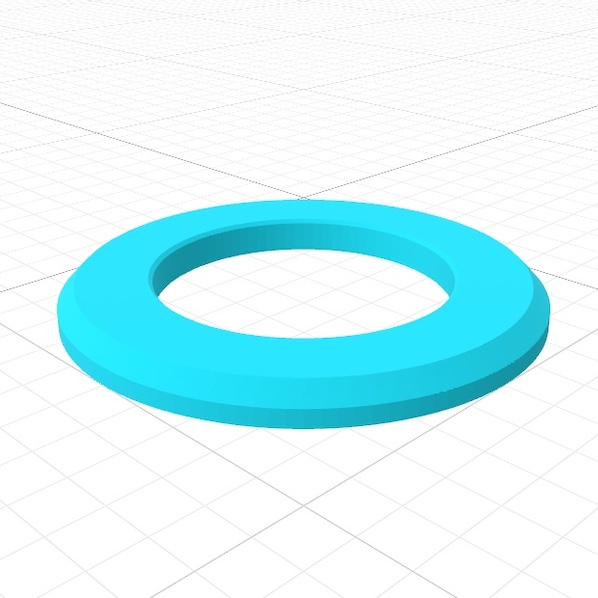}
     \end{subfigure}
     \begin{subfigure}[b]{0.10\textwidth}
         \centering
         \includegraphics[width=\textwidth]{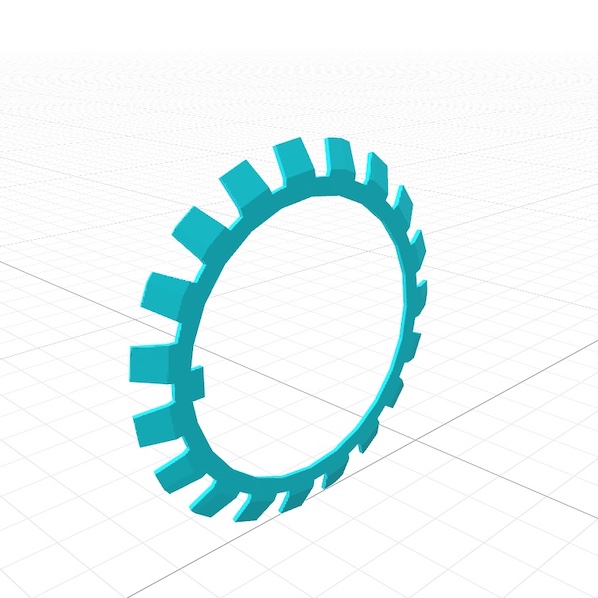}
     \end{subfigure}
     \\
     \begin{subfigure}[b]{0.10\textwidth}
         \centering
         \includegraphics[width=\textwidth]{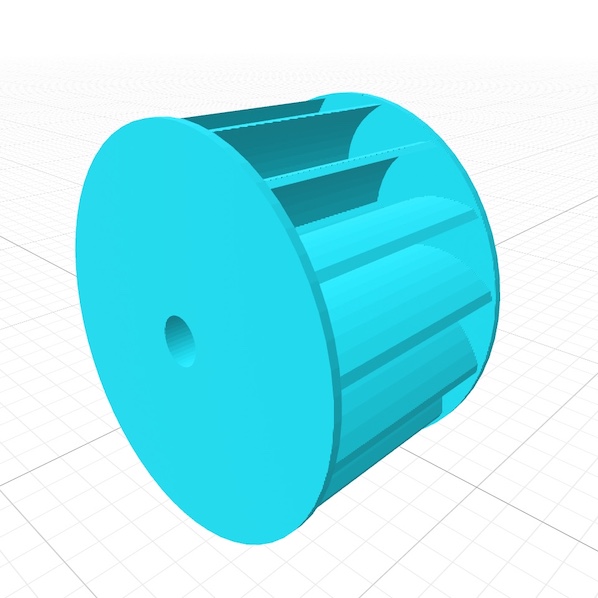}
     \end{subfigure}
     \begin{subfigure}[b]{0.10\textwidth}
         \centering
         \includegraphics[width=\textwidth]{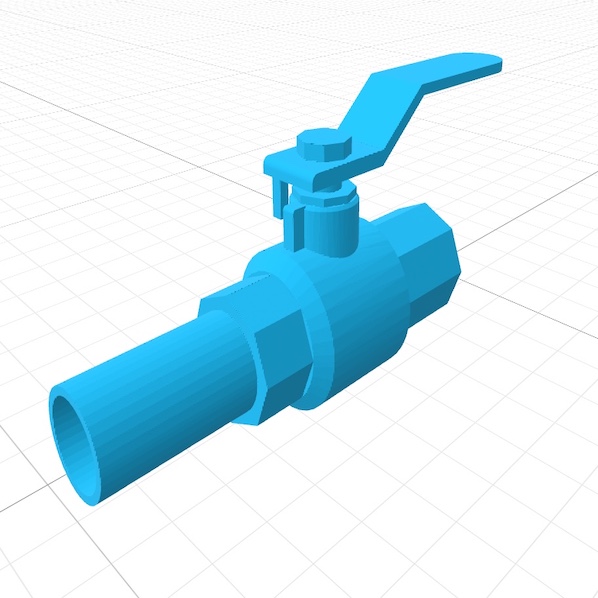}
     \end{subfigure}
    \begin{subfigure}[b]{0.10\textwidth}
         \centering
         \includegraphics[width=\textwidth]{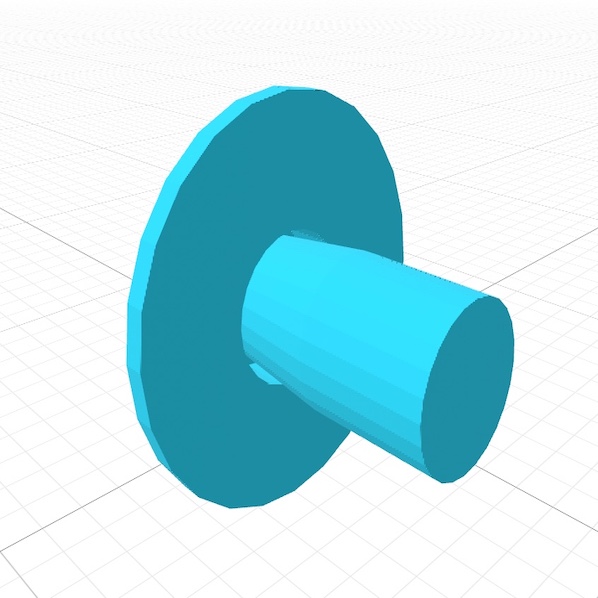}
     \end{subfigure}
    \begin{subfigure}[b]{0.10\textwidth}
         \centering
         \includegraphics[width=\textwidth]{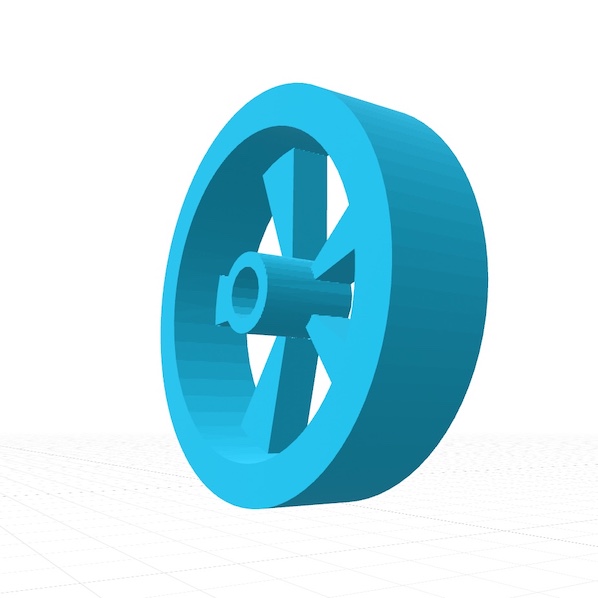}
     \end{subfigure}
    \caption{Shape of industrial components in the dataset}
    \label{fig:all_kinds}
\end{figure}

\subsection{Characteristic of industrial components}

\begin{figure}[H]
    \centering
    \includegraphics[width=50mm,scale=0.5]{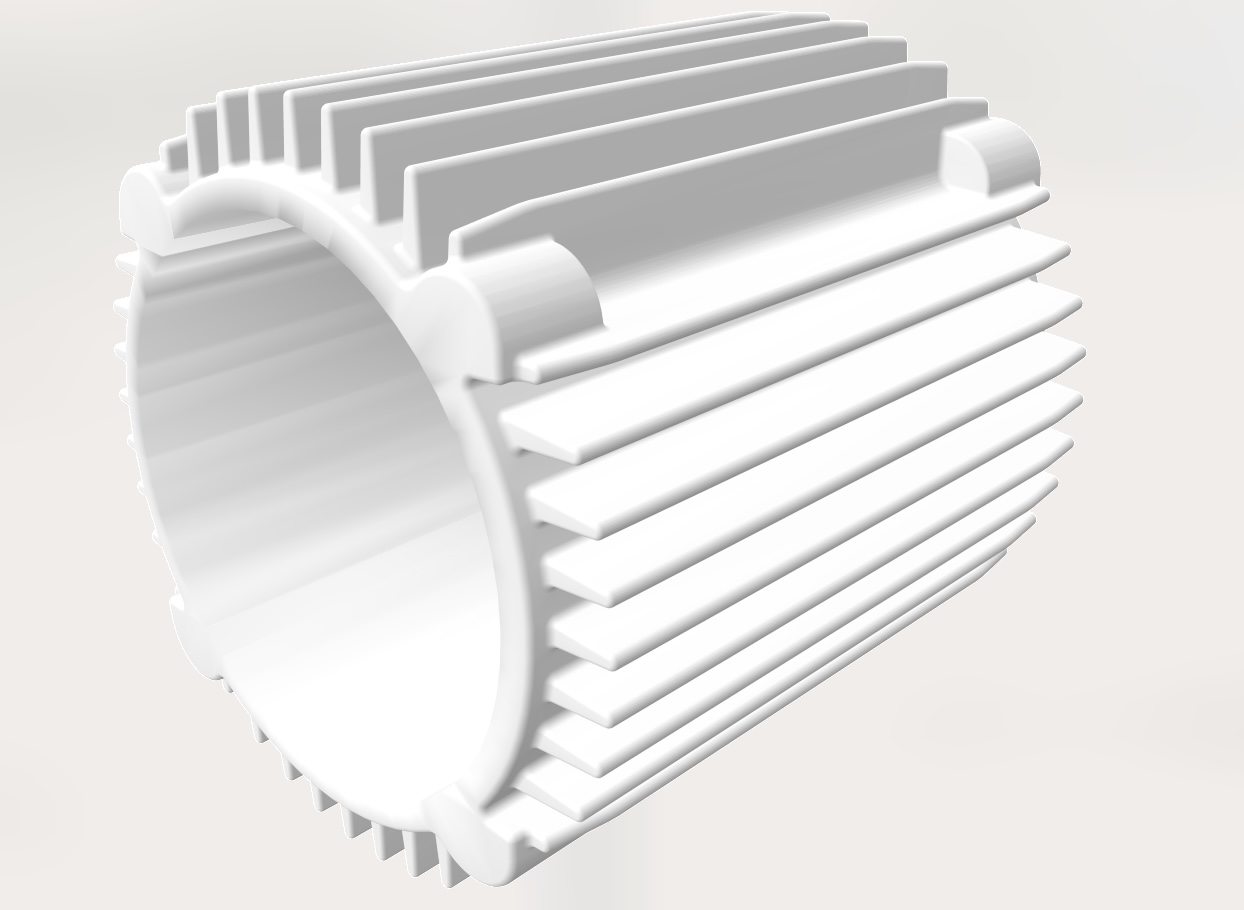}
    \caption{Motor image}
    \label{fig:motor_label}
\end{figure}

From the figure \ref{fig:motor_label} above, we estimate how many parameters a specific object has in our dataset. Notice that although the dataset has different types of mechanical components, the parameters should be roughly the same. From the motor image above, we first measure the inner and outer circle radii. Besides, we also measure the height of the hollow cylinder. The circle's edge is also not completely rigid, so we add another curvature parameter. For the gear of the motor, we notice different kinds of gear on the side. There are two kinds of gear on the motor: one can be approximated as part of a circle and thus has three parameters: the circle's radius, the angle of the arc, and how much of the arc part is shown. Also, some of the gears on the circle can be modeled as a triangle. Therefore, we obtain the number of triangles, the base length, and the triangle's height. In total, there would be around ten parameters for a specific object. To use some abstract variables to capture objects' properties, we propose using a roughly similar number of variables to describe them. Past literature about grasping visual features on 3D objects ~\cite{weinmann2013visual} illustrates the importance of eigenvalues in defining corners in corner detectors and deciding other features of the object. Specifically, after observing the structure of a 3D point cloud of specific mechanical components and look at past literature ~\cite{li2019dl} ~\cite{huang2023weld} ~\cite{weinmann2013feature} ~\cite{li2020geometry}, we decide that including linearity $L$, planarity $P$, sphericity $S$, omni-variance $O$, anisotropy $A$, the sum of eigenvalues $\Sigma$ and curvature $C$ would be necessary. 

We can express the values above by using the eigenvalues of the neighborhood's covariance matrix. The dimension of the covariance matrix would be $3 x 3$, so there would be three eigenvalues. Denote them as $\lambda_1 \ge \lambda_2 \ge \lambda_3 \ge 0$ as the eigenvalues of the covariance matrix are positive. Then we have

\begin{align*}
    L &= \dfrac{\lambda_1 - \lambda_2}{\lambda_1} \le 1\\
    P &= \dfrac{\lambda_2 - \lambda_3}{\lambda_1}\\
    S &= \dfrac{\lambda_3}{\lambda_1} \le 1\\
    A &= \dfrac{\lambda_1 - \lambda_3}{\lambda_1}\le 1\\
    \Sigma &= \lambda_1 + \lambda_2 + \lambda_3 \\
    C &= \dfrac{\lambda_3}{\lambda_1 + \lambda_2 + \lambda_3}\\
    O &= \sqrt[3]{\lambda_1 * \lambda_2 * \lambda_3}
\end{align*}

For instance, for the figure of the point cloud data of one motor above, we should depict the sphericity of the motor and the curvature of the curves.

\subsection{Data processing}
\label{sub:data_processing}

We use the blender python module to process the dataset \url{https://pypi.org/project/bpy/}. 
The reason behind processing the dataset before training is that there will be a different number of points. To make the number of points uniform for the model to train without changing the structure, we either add points to the 
3d object or take a subset of the points out from the object. We decide to use 1024 points of the object. The latter is easier to do as we can shuffle the points and take the first 1024 points. For the first one, we use the following algorithm: we first get the list of surfaces as triangles from the blender and then calculate the number of points we add. Suppose now we have $n$ points in the objects, and there are $t$ triangles. Then, we add $\dfrac{1024 - n}{t}$ on each triangle. If $t$ does not divide $1024 - n$, we can add more points on each triangle and sample points from the object to take points out. The main idea of our preprocessing is to ensure the object has enough points to get sent into the model while making sure the object differs a little from the original object.

\section{Methodology}

In this section, we discuss the structure of the model. We propose an approach that gives a geometric embedding to the points first and then uses hierarchies of graph-like convolution that preserve invariance to process the points.

\subsection{Geometric pre-processing}

As discussed in section \ref{sub:data_processing}, we captured the core geometric features of an object by calculating around ten variables using the covariance matrix of each neighborhood. The reason for calculating the variables only in each neighborhood is to capture the local feature of the model. There are many ways to sample the neighborhood, led by the method of farthest point sampling ~\cite{UCAM-CL-TR-562} in a 3D metric space with euclidean distance with upper bound radius $r$. In the neighborhood, the center is point $c$, the neighbor points $p_i \in N(c)$, and we strive to use geometric features to get a local pattern representation $f$ for this neighborhood. Each point $p_i$ has its feature, and we try to process that first and then do aggregation:

\[
   f = A(F(f_{p_i}), \forall p_i)
\]

where $f_{p_i}$ is the feature of point $p_i$. In the scope of this paper, $p_i$ is the 3D coordinates of the point $p_i$ expressed in a vector.

We concatenate all the variables into a vector $G_v$, which is:

\[
    G_v = [L, P, S, A, \Sigma, C, O]
\]

After calculating the geometric features, we concatenate the feature vector with the original vector, which results in

\[
    F(f_{p_i}) = [\text{MLP}(G_v), f_{p_i}]
\]


\subsection{Graph Neural network}

Considering the properties of the points in the point cloud, we treat the whole point cloud as a Graph Neural Network (GNN). To set up, we first compute a directed graph $G = (V, E)$, representing the cloud structure. There are many ways to compute the graph out of the given vertices, including the most basic one: computer a k-nearest-neighbor (k-NN) graph, where k is a hyperparameter to be tuned for the vertices. In each neighborhood, we connect pairs of vertices inside it. Therefore, a vertex will get connected to itself, and such self-loops in the graph are fine. We denote the set of vertices as $V$ and edges as $E \subseteq V \times V$. 

Now we define the edge feature for this graph, as this defines the important relationship between points in the point cloud. For two points $p_i, p_j$ in the point cloud, we use a nonlinear function with learnable parameters, $f$, to get the edge feature. In order words, the edge feature $e_{ij}$ between points $p_i, p_j$ is equal to 

\[
e_{ij} = f(p_i, p_j)
\].

Then, we define edge convolution, which aggregates the edge features with all edges from a vertex. The equation is just applying an aggregation function on it, i.e., the feature $x_i$ for point $i$ is

\[
    x_i' = g_{j: (i, j) \in E} f(p_i, p_j)
\]

There are several choices for $g$ and $f$. However, considering we want to keep the function simple while capturing the global and local information of the point cloud. To express the global information, we use the point's value, $p_i$. For the local information, on the other hand, we use the difference between the values of points, $p_i - p_j$, to express the local information. Adding in the aggregation function, we have the edge convolution has the form:

\[
x_im' = \max_{j: (i, j) \in E} \text{ReLU} (\theta_{1_m} \cdot (p_j - p_i) + \theta_{2_m} \cdot x_i)
\]

\subsection{Dynamic updating}

We treat the edge features as a graph to continue applying edge convolution to increase the receptive field's size. Specifically, suppose at layer $i$ we have the $k$ points in the neighbourhood around point $p_x$ are $p_{c1}, p_{c2}, \cdots, p_{ck}$. Then in layer $i+1$, we have the corresponding edges are 

\[
(i, c1), (i, c2), (i, c3), \cdots, (i, ck)
\].

In other words, we treat the new edge features computed as points in the new graph. The reason why it would increase the receptive field is that the distance of "points" are closer and closer when the layer progresses: the "points" in a new layer are the edges from the previous layer, and points that are not connected in the previous layer might be related since features involving them will be connected in the new layer. As a result, the size of the receptive field for this model might ideally extend to the diameter of the point cloud, i.e., the maximum distance between points in the point cloud.

\subsection{Proof of invariances}

This section proves that the function we have found before can process point clouds. That is to say, the function we have picked before to process the data at each layer would satisfy \textbf{Permutation Invariance}  and \textbf{Translation Invariance}.

\begin{theorem}
    The output for every layer, 

    \[
       f_i' = \max_{j: (i, j) \in E} h_{\theta} (x_i, x_j) 
    \]

    is permutation invariant. 

    That is to say, given a set of points $p_1, p_2, \cdots, p_n$, after rearranging it into $p_1', p_2', \cdots, p_n'$, the result would still be the same.
\end{theorem}

\begin{proof}
    The permutation invariance is achieved since the max function is symmetric. 
\end{proof}

For translation invariance, suppose the points are shifted by adding a constant $c$. Then we compute the edge feature between $x_i$ and $x_j$ with both terms replaced with $x_i + c$, $x_j + c$ respectively:

\begin{align*}
   e_mij'  &= \text{ReLU}(\theta_{1_m} \cdot ((x_j + c) - (x_i + c)+\theta_{2_m} \cdot (x_i + c))  \\
           &= \text{ReLU} (\theta_{1_m} \cdot (x_j - x_i) + \theta_{2_m} \cdot (x_i + c)
\end{align*}

Therefore, the first term in the equation above inside the ReLU function has translation invariance, while the second term does not. Therefore, our model has partial translation invariance. If the model achieves full translation invariance, the trainable function $\theta_{2_m}$ above is zero. However, it means our model would depend only on the relative positions of points represented by the term $x_j - x_i$ and ignore the true position of $x_i$. As a result, the model would only process the input of an unordered set of parts of the object and not its orientation and positions. Since the parameters for $x_i + c$ are small enough, it would not propagate much after layers of computation.

\subsection{Performance of the model}

This section discusses the model's performance based on visualizations such as the confusion matrix and AUC curve. 

We first discuss the confusion matrix and our model's performance on it. The confusion matrix shows how well our model performs in supervised learning, specifically in classifying the geometrical objects, in this case, inside the confusion matrix, which is a big square matrix whose number of columns (or rows) matches the number of types for classification. Each row of the matrix represents an instance of an actual class, while each column represents an instance of the predicted class. The darker the color in a grid, the more instances there are. The ideal performance of a model would be that the diagonal of the matrix would be very dark, meaning there are fewer false positive or false negative samples.

To generate the confusion matrix on our dataset, we calculate the prediction on each label and normalize it by dividing the prediction by the total number of predictions. The confusion matrix looks like the one below, using grayscale as color.

\begin{figure}[H]
    \centering
    \includegraphics[width=0.8\textwidth]{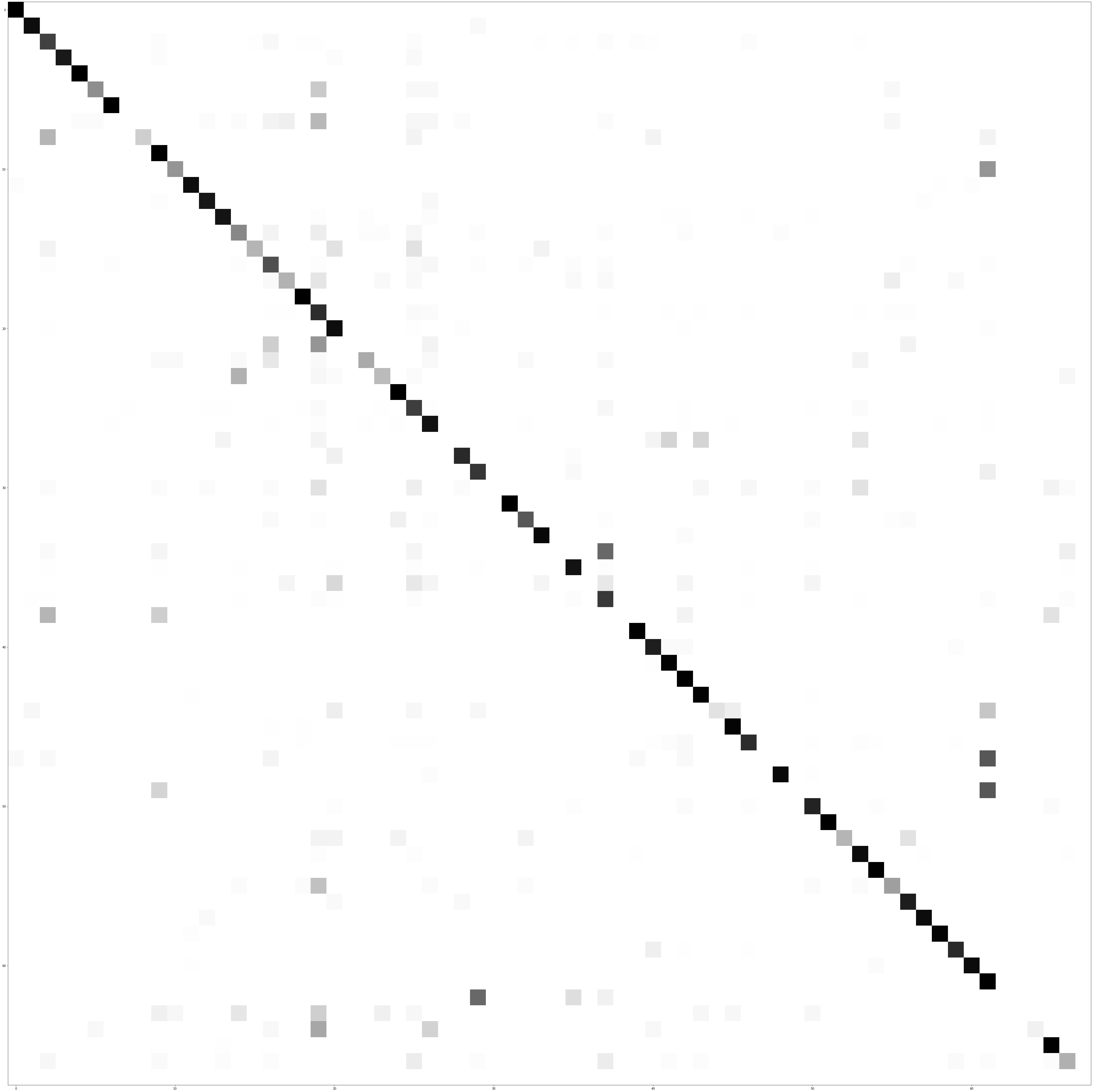}
    \caption{Confusion matrix for our model}
    \label{fig:confusion_matrix}
\end{figure}

As seen in the confusion matrix diagram above, there is a pattern consisting of self-explanatory dark color blocks on the diagonal, which means the model achieves a high true positive rate for most types. The color of blocks for other regions in the matrix is not as dark as the ones on the diagonal, which means our model also achieves a low false positive/false negative rate on most labels.

Then, we also show the strength of our model by displaying the ROC (receiver operating characteristic) curve to illustrate how well our model is in diagnosing true types from false types. Specifically, in evaluating our model, we generate a plot for each possible label in our dataset. We aggregate two data types for each label $i$: the expected actual label is $i$, or the predicted label for the object is $i$. The former is for true positive and false positive, and the latter is for true negative and false negative. We generate a ROC curve for each of the labels, which can be seen in the figures in the following pages:

\begin{figure}[H]
     \centering
      \begin{subfigure}[b]{0.33\textwidth}
         \centering
         \includegraphics[width=\textwidth]{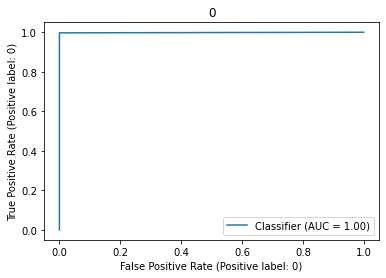}
     \end{subfigure}
     \begin{subfigure}[b]{0.33\textwidth}
         \centering
         \includegraphics[width=\textwidth]{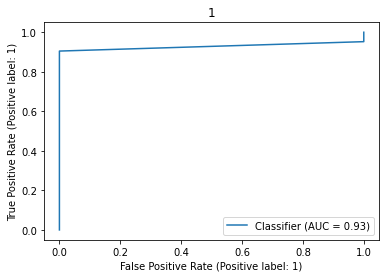}
     \end{subfigure}
     \begin{subfigure}[b]{0.33\textwidth}
         \centering
         \includegraphics[width=\textwidth]{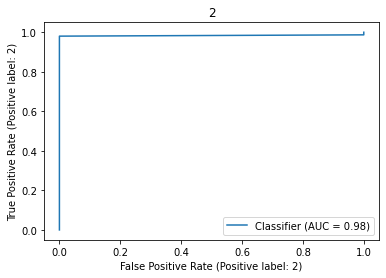}
     \end{subfigure}
     \\
     \begin{subfigure}[b]{0.33\textwidth}
         \centering
         \includegraphics[width=\textwidth]{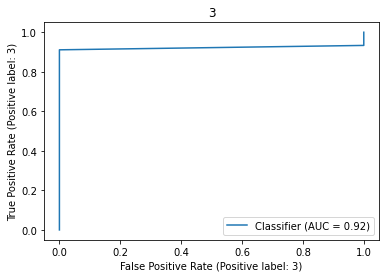}
     \end{subfigure}
     \begin{subfigure}[b]{0.33\textwidth}
         \centering
         \includegraphics[width=\textwidth]{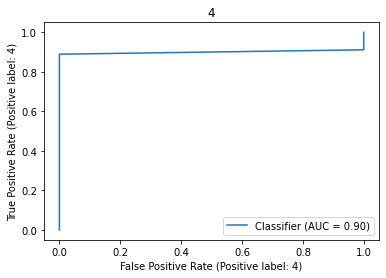}
     \end{subfigure}
    \begin{subfigure}[b]{0.33\textwidth}
         \centering
         \includegraphics[width=\textwidth]{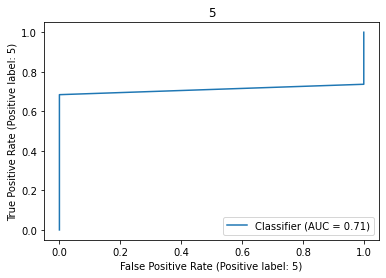}
     \end{subfigure}
     \\
    \begin{subfigure}[b]{0.33\textwidth}
         \centering
         \includegraphics[width=\textwidth]{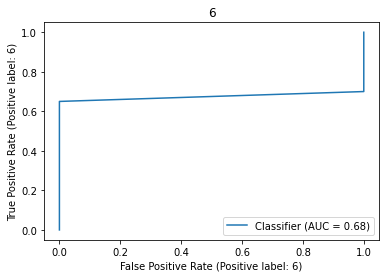}
     \end{subfigure}
    \begin{subfigure}[b]{0.33\textwidth}
         \centering
         \includegraphics[width=\textwidth]{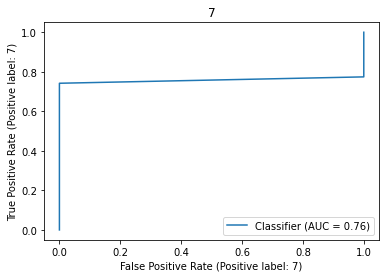}
     \end{subfigure}
    \begin{subfigure}[b]{0.33\textwidth}
         \centering
         \includegraphics[width=\textwidth]{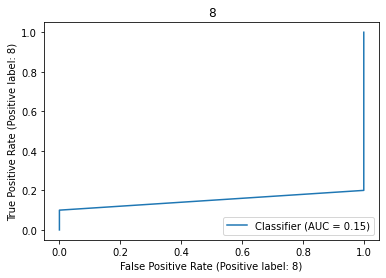}
     \end{subfigure}   
     \\
    \begin{subfigure}[b]{0.33\textwidth}
         \centering
         \includegraphics[width=\textwidth]{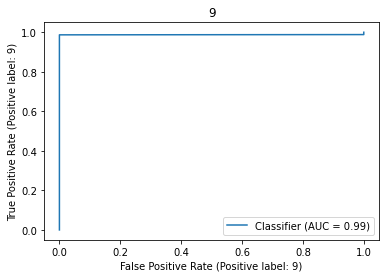}
     \end{subfigure}     
    \begin{subfigure}[b]{0.33\textwidth}
         \centering
         \includegraphics[width=\textwidth]{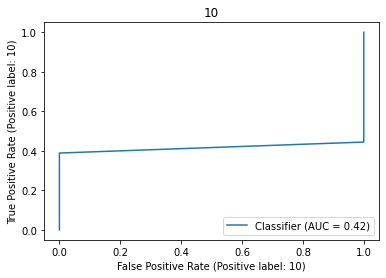}
     \end{subfigure}     
          \begin{subfigure}[b]{0.33\textwidth}
         \centering
         \includegraphics[width=\textwidth]{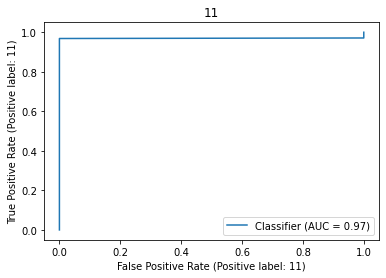}
     \end{subfigure}
     \\
     \begin{subfigure}[b]{0.33\textwidth}
         \centering
         \includegraphics[width=\textwidth]{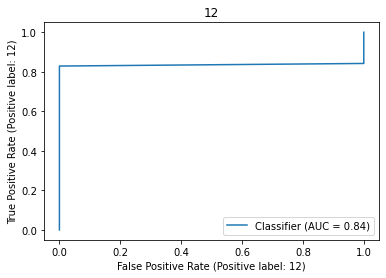}
     \end{subfigure}
     \begin{subfigure}[b]{0.33\textwidth}
         \centering
         \includegraphics[width=\textwidth]{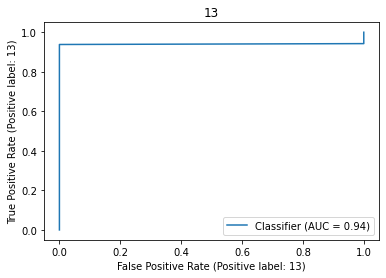}
     \end{subfigure}
     \begin{subfigure}[b]{0.33\textwidth}
         \centering
         \includegraphics[width=\textwidth]{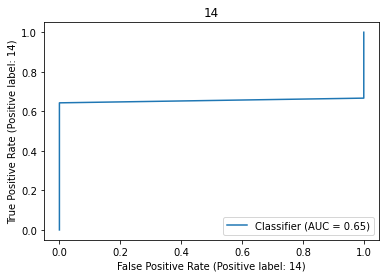}
     \end{subfigure}
    \\
    \begin{subfigure}[b]{0.33\textwidth}
         \centering
         \includegraphics[width=\textwidth]{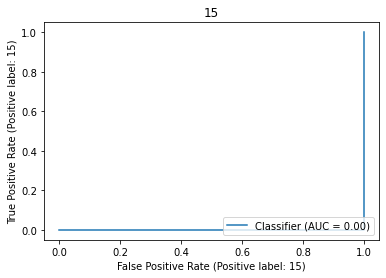}
     \end{subfigure}
    \begin{subfigure}[b]{0.33\textwidth}
         \centering
         \includegraphics[width=\textwidth]{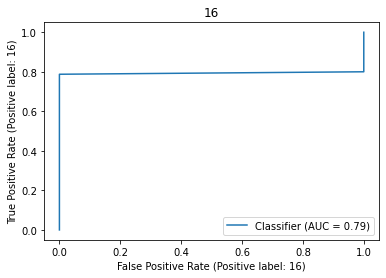}
     \end{subfigure}
    \begin{subfigure}[b]{0.33\textwidth}
         \centering
         \includegraphics[width=\textwidth]{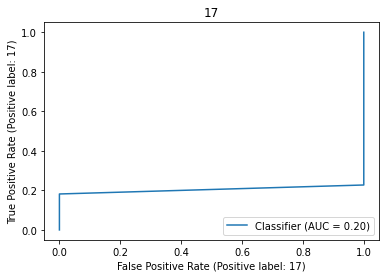}
     \end{subfigure}
     \caption{Figure of ROC for all the categories prediction}
    \label{fig:all_kinds_ROC}
\end{figure}

\begin{figure}[H]
    \ContinuedFloat
    \begin{subfigure}[b]{0.33\textwidth}
         \centering
         \includegraphics[width=\textwidth]{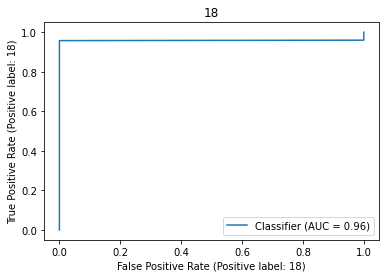}
     \end{subfigure}   
    \begin{subfigure}[b]{0.33\textwidth}
         \centering
         \includegraphics[width=\textwidth]{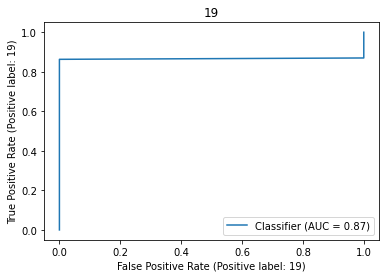}
     \end{subfigure}     
    \begin{subfigure}[b]{0.33\textwidth}
         \centering
         \includegraphics[width=\textwidth]{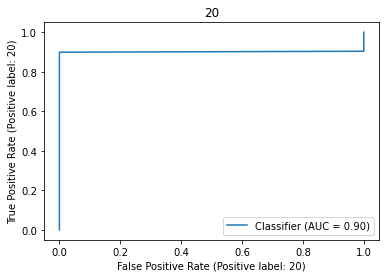}
     \end{subfigure}     
    //
          \begin{subfigure}[b]{0.33\textwidth}
         \centering
         \includegraphics[width=\textwidth]{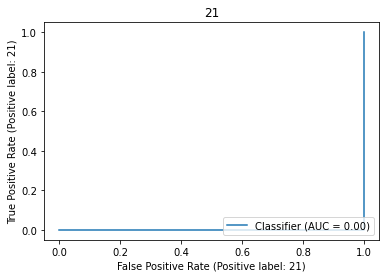}
     \end{subfigure}
     \begin{subfigure}[b]{0.33\textwidth}
         \centering
         \includegraphics[width=\textwidth]{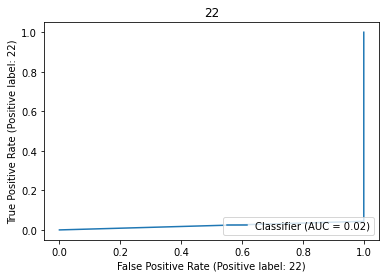}
     \end{subfigure}
     \begin{subfigure}[b]{0.33\textwidth}
         \centering
         \includegraphics[width=\textwidth]{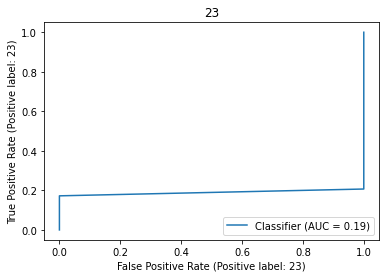}
     \end{subfigure}
     \\
     \begin{subfigure}[b]{0.33\textwidth}
         \centering
         \includegraphics[width=\textwidth]{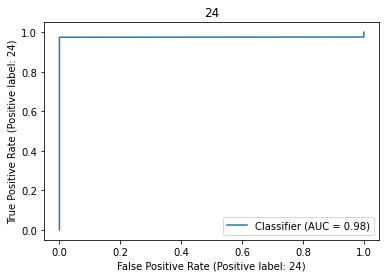}
     \end{subfigure}
    \begin{subfigure}[b]{0.33\textwidth}
         \centering
         \includegraphics[width=\textwidth]{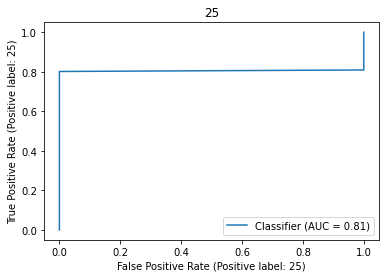}
     \end{subfigure}
    \begin{subfigure}[b]{0.33\textwidth}
         \centering
         \includegraphics[width=\textwidth]{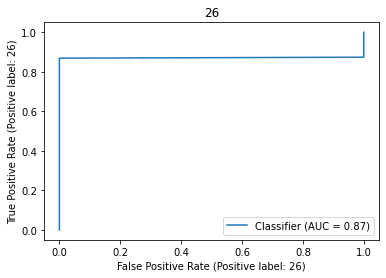}
     \end{subfigure}
     \\
    \begin{subfigure}[b]{0.33\textwidth}
         \centering
         \includegraphics[width=\textwidth]{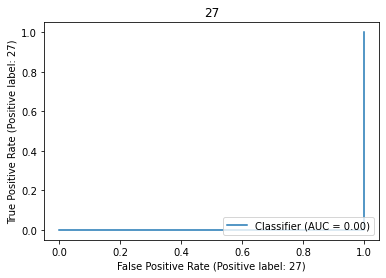}
     \end{subfigure}
    \begin{subfigure}[b]{0.33\textwidth}
         \centering
         \includegraphics[width=\textwidth]{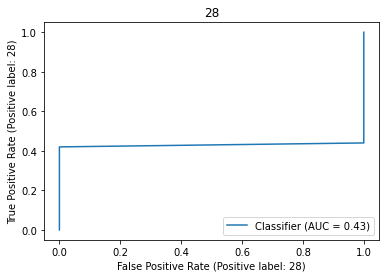}
     \end{subfigure}   
    \begin{subfigure}[b]{0.33\textwidth}
         \centering
         \includegraphics[width=\textwidth]{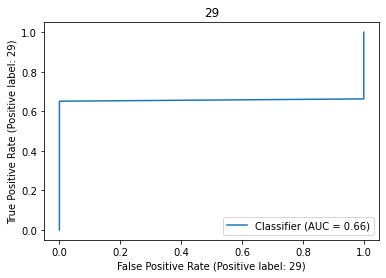}
     \end{subfigure}  
     \\
    \begin{subfigure}[b]{0.33\textwidth}
         \centering
         \includegraphics[width=\textwidth]{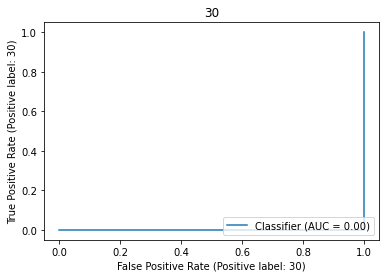}
     \end{subfigure}     
          \begin{subfigure}[b]{0.33\textwidth}
         \centering
         \includegraphics[width=\textwidth]{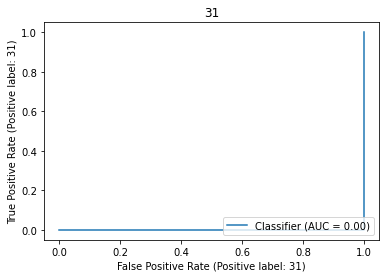}
     \end{subfigure}
     \begin{subfigure}[b]{0.33\textwidth}
         \centering
         \includegraphics[width=\textwidth]{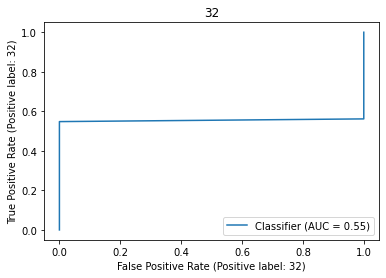}
     \end{subfigure}
     \\
     \begin{subfigure}[b]{0.33\textwidth}
         \centering
         \includegraphics[width=\textwidth]{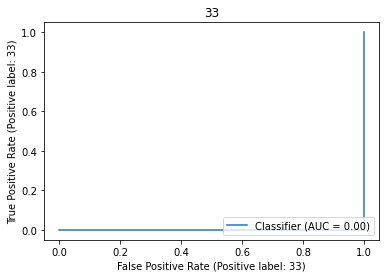}
     \end{subfigure}
     \begin{subfigure}[b]{0.33\textwidth}
         \centering
         \includegraphics[width=\textwidth]{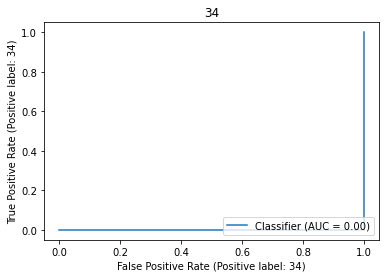}
     \end{subfigure}
    \begin{subfigure}[b]{0.33\textwidth}
         \centering
         \includegraphics[width=\textwidth]{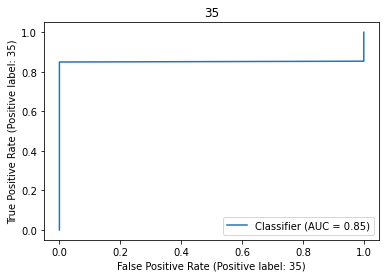}
     \end{subfigure}
     \\
    \caption{Figure of ROC for all the categories prediction (continued)}
    \label{fig:all_kinds_ROC}
\end{figure}

\begin{figure}[H]
    \centering
    \begin{subfigure}[b]{0.33\textwidth}
         \centering
         \includegraphics[width=\textwidth]{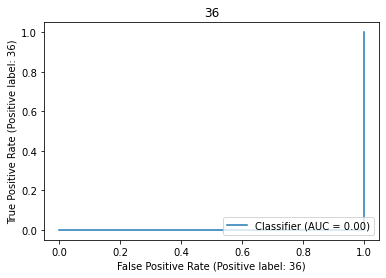}
     \end{subfigure}
    \begin{subfigure}[b]{0.33\textwidth}
         \centering
         \includegraphics[width=\textwidth]{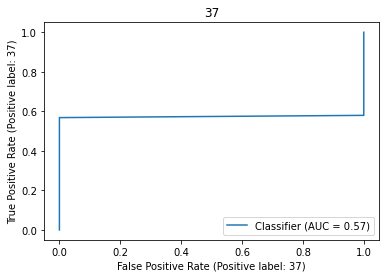}
     \end{subfigure}
    \begin{subfigure}[b]{0.33\textwidth}
         \centering
         \includegraphics[width=\textwidth]{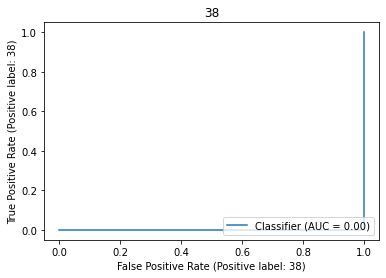}
     \end{subfigure}  
     \\
    \begin{subfigure}[b]{0.33\textwidth}
         \centering
         \includegraphics[width=\textwidth]{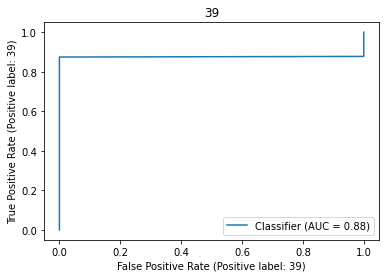}
     \end{subfigure}    
    \begin{subfigure}[b]{0.33\textwidth}
         \centering
         \includegraphics[width=\textwidth]{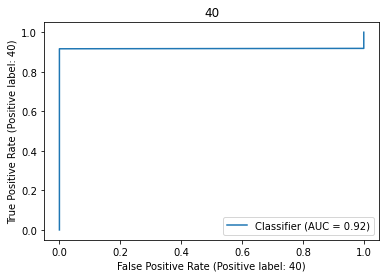}
     \end{subfigure}     
          \begin{subfigure}[b]{0.33\textwidth}
         \centering
         \includegraphics[width=\textwidth]{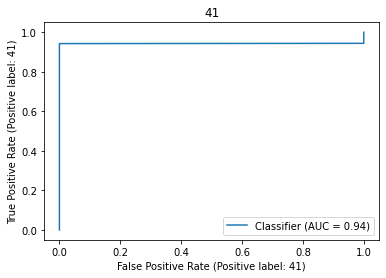}
     \end{subfigure}
     \\
     \begin{subfigure}[b]{0.33\textwidth}
         \centering
         \includegraphics[width=\textwidth]{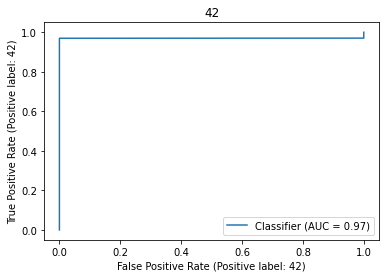}
     \end{subfigure}
     \begin{subfigure}[b]{0.33\textwidth}
         \centering
         \includegraphics[width=\textwidth]{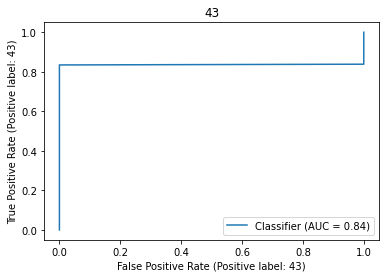}
     \end{subfigure}
     \begin{subfigure}[b]{0.33\textwidth}
         \centering
         \includegraphics[width=\textwidth]{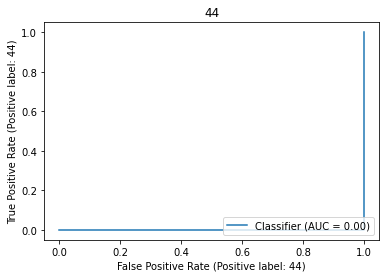}
     \end{subfigure}
     \\
    \begin{subfigure}[b]{0.33\textwidth}
         \centering
         \includegraphics[width=\textwidth]{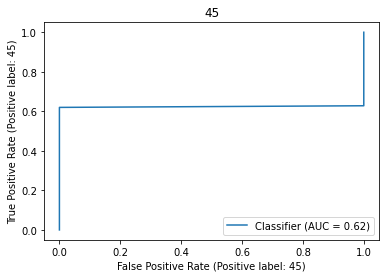}
     \end{subfigure}
    \begin{subfigure}[b]{0.33\textwidth}
         \centering
         \includegraphics[width=\textwidth]{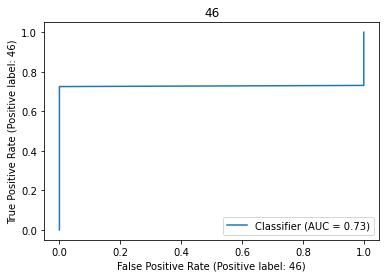}
     \end{subfigure}
    \begin{subfigure}[b]{0.33\textwidth}
         \centering
         \includegraphics[width=\textwidth]{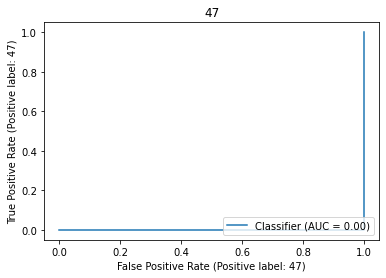}
     \end{subfigure}
     \\
    \begin{subfigure}[b]{0.33\textwidth}
         \centering
         \includegraphics[width=\textwidth]{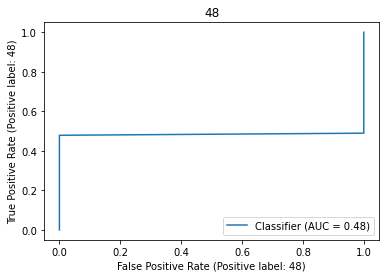}
     \end{subfigure}   
    \begin{subfigure}[b]{0.33\textwidth}
         \centering
         \includegraphics[width=\textwidth]{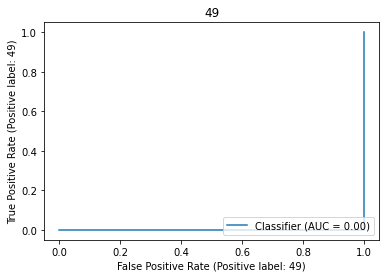}
     \end{subfigure}     
    \begin{subfigure}[b]{0.33\textwidth}
         \centering
         \includegraphics[width=\textwidth]{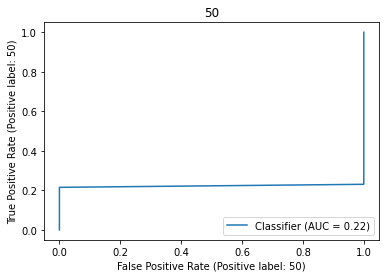}
     \end{subfigure}     
    \\
          \begin{subfigure}[b]{0.33\textwidth}
         \centering
         \includegraphics[width=\textwidth]{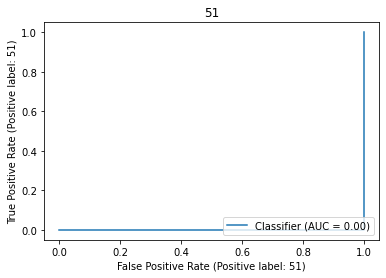}
     \end{subfigure}
     \begin{subfigure}[b]{0.33\textwidth}
         \centering
         \includegraphics[width=\textwidth]{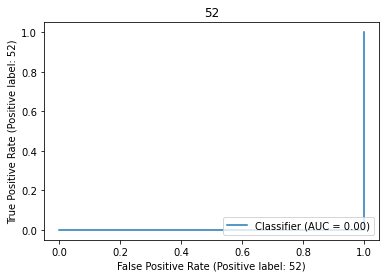}
     \end{subfigure}
     \begin{subfigure}[b]{0.33\textwidth}
         \centering
         \includegraphics[width=\textwidth]{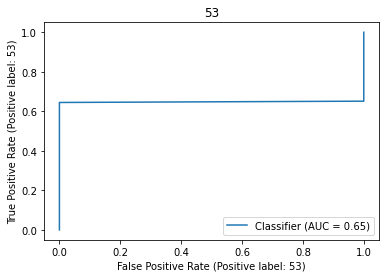}
     \end{subfigure}
     \\
    \caption{Figure of ROC for all the categories prediction (continued)}
    \label{fig:all_kinds_ROC}
\end{figure}

\begin{figure}[H]
    \centering 
     \begin{subfigure}[b]{0.33\textwidth}
         \centering
         \includegraphics[width=\textwidth]{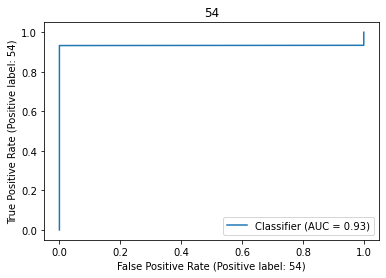}
     \end{subfigure}
    \begin{subfigure}[b]{0.33\textwidth}
         \centering
         \includegraphics[width=\textwidth]{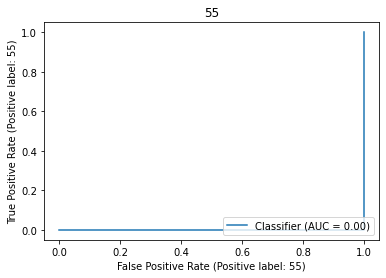}
     \end{subfigure}
    \begin{subfigure}[b]{0.33\textwidth}
         \centering
         \includegraphics[width=\textwidth]{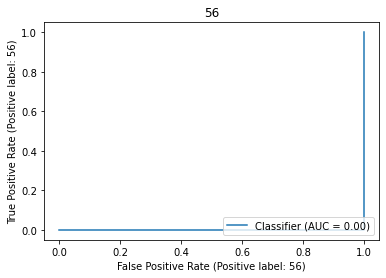}
     \end{subfigure}
     \\
    \begin{subfigure}[b]{0.33\textwidth}
         \centering
         \includegraphics[width=\textwidth]{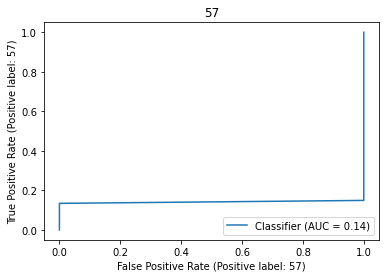}
     \end{subfigure}
    \begin{subfigure}[b]{0.33\textwidth}
         \centering
         \includegraphics[width=\textwidth]{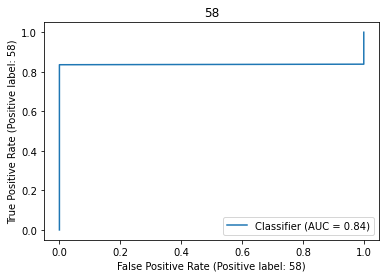}
     \end{subfigure}   
    \begin{subfigure}[b]{0.33\textwidth}
         \centering
         \includegraphics[width=\textwidth]{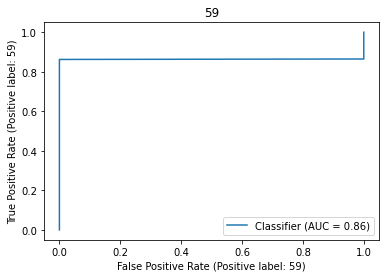}
     \end{subfigure}     
     \\
    \begin{subfigure}[b]{0.33\textwidth}
         \centering
         \includegraphics[width=\textwidth]{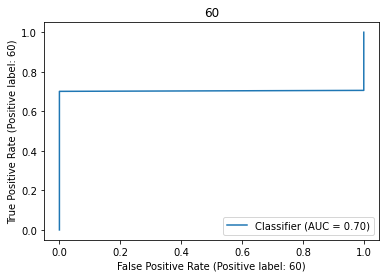}
     \end{subfigure}     
          \begin{subfigure}[b]{0.33\textwidth}
         \centering
         \includegraphics[width=\textwidth]{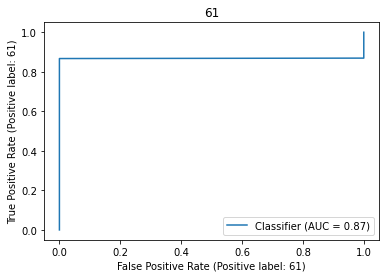}
     \end{subfigure}
     \begin{subfigure}[b]{0.33\textwidth}
         \centering
         \includegraphics[width=\textwidth]{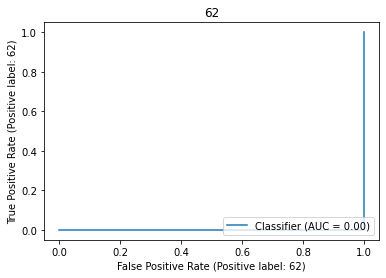}
     \end{subfigure}
     \begin{subfigure}[b]{0.33\textwidth}
         \centering
         \includegraphics[width=\textwidth]{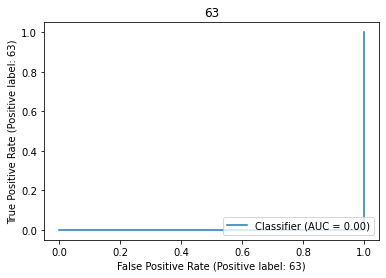}
     \end{subfigure}
     \begin{subfigure}[b]{0.33\textwidth}
         \centering
         \includegraphics[width=\textwidth]{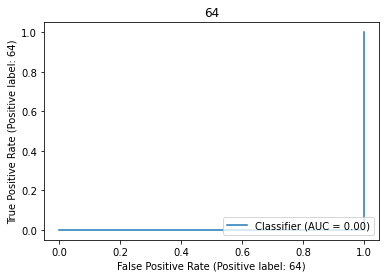}
     \end{subfigure}
    \begin{subfigure}[b]{0.33\textwidth}
         \centering
         \includegraphics[width=\textwidth]{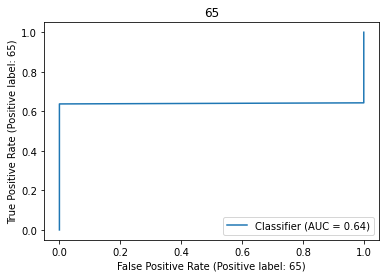}
     \end{subfigure}
    \begin{subfigure}[b]{0.33\textwidth}
         \centering
         \includegraphics[width=\textwidth]{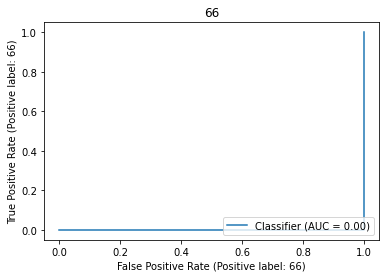}
     \end{subfigure}
    \caption{Figure of ROC for all the categories prediction (continued)}
    \label{fig:all_kinds_ROC}
\end{figure}

From the plots above, we can see that for most of the categories for mechanical components in our dataset, the curve is always above the diagonal line by the equation $y = x$, and the AUC (area under the ROC curve) reaches high value quickly. From that, our model does distinguish each type of mechanical component from others.

\subsection{Comparison to PointNet}

In this section, we compare our model and one of the essential models in point cloud classification, PointNet. First, we consider the number of parameters and the number in the PointNet model. By counting the numbers in the parameters attribute of both models in PyTorch, we discovered that our model has 1,816,515 parameters, while PointNet has 1,621,325 parameters. Therefore, the magnitude of the parameters of both models is similar. In contrast, the comparison of the model performance below shows that our model performs better by considering mode levels of information, specifically both global and local information of 3D data.

By the graph below, we can see that our model achieves higher accuracy than that of PointNet. As we can see, our model reaches higher accuracy within a reasonable number of epochs.
\begin{figure}
    \centering
    \includegraphics[width=0.4\textwidth]{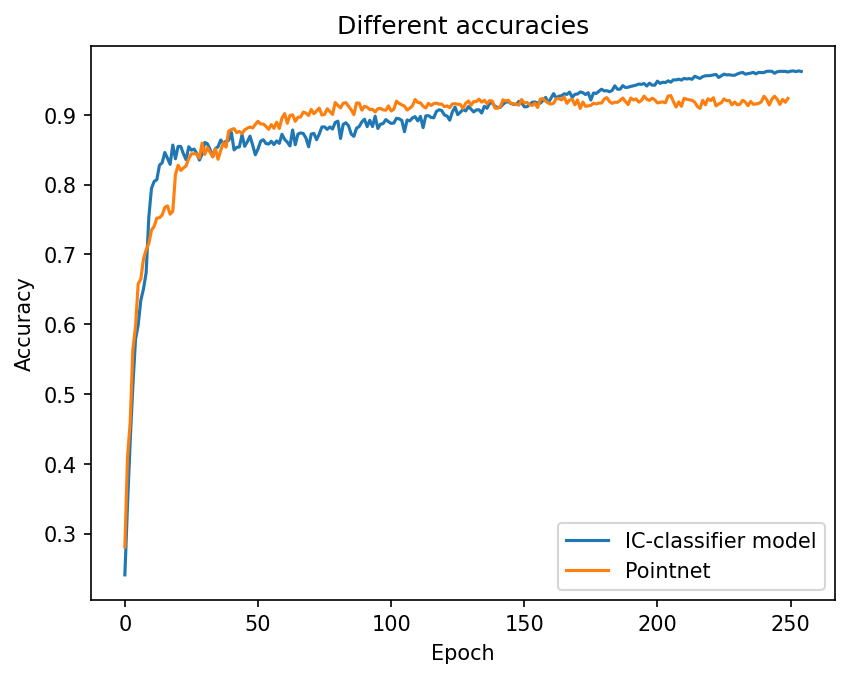}
    \caption{comparison between our model and PointNet}
    \label{fig:comparison_label}
\end{figure}

\subsection{Summary of result}
We run the model with hyperparameter $k$ being 20 with Stochastic gradient descent (SGD) and take a look at its performance at classifying objects, telling true positive labels from others, and we compare its performance with other state-of-the art models and see that our model can reach similar performance as well.

\section{Further work}

In this work, we have delivered a novel framework that could effectively identify mechanical components that outperform the current state-of-the-art models in identifying the type of mechanical components. However, there is still room for improvement in this framework. For instance, alternative approaches, such as capturing the essential quantities of the mechanical components into a bottleneck, might bring different performances. Also, the network structure could be tuned further to ensure better performance.

\section{Conclusion}

As seen in this paper, the model incorporates both the ability to identify global and local features of mechanical components by including geometric before embedding the points and a graphical neural network model that repeatedly computes the features of the point cloud as a graph and aggregates them. As displayed above, the model performs similarly to the state-of-the-art models in classifying mechanical components.  

\section*{Acknowledgements}

This work is supported by the National Natural Science Foundation of China under Grant 52175237.

\bibliographystyle{unsrt}  
\bibliography{IC_bib} 

\end{document}